\documentclass[10pt]{article}
\usepackage[a4paper,margin=0.85in]{geometry}
\usepackage{amsmath}
\usepackage{amssymb}
\usepackage{amsthm}
\usepackage{mathtools}
\usepackage{microtype}
\usepackage{times}
\usepackage{bm}
\usepackage{mathrsfs}
\usepackage{natbib}
\usepackage[hidelinks]{hyperref}
\newcommand{\Beginproof}[1]{\begin{proof}\ignorespaces#1}
\newcommand{\Endproof}{\end{proof}}

\newtheorem{theorem}{Theorem}
\newtheorem{lemma}[theorem]{Lemma}
\newtheorem{proposition}[theorem]{Proposition}
\newtheorem{corollary}[theorem]{Corollary}
\newtheorem{assumption}[theorem]{Assumption}
\theoremstyle{remark}
\newtheorem{remark}[theorem]{Remark}
\newenvironment{keywords}
  {\par\smallskip\noindent\textbf{Keywords: }}
  {\par\smallskip}

\newcommand{\E}{\mathbb E}
\newcommand{\R}{\mathbb R}

\newcommand{\tr}{\operatorname{tr}}
\newcommand{\norm}[1]{\left\|#1\right\|}
\newcommand{\polybrac}[1]{\langle\mkern-4.5mu\langle #1\rangle\mkern-4.5mu\rangle}

\title{A Rate Transfer Method: Convergence Rates in Stochastic Gradient Descent via Stochastic Calculus}
\author{Neil Walton\thanks{\href{mailto:neil.walton@durham.ac.uk}{neil.walton@durham.ac.uk}}\\Durham University}
\date{}
\begin{document}
\maketitle
\begin{abstract}
We develop a rate-transfer method that converts finite-time Lyapunov estimates for an associated stochastic differential equation (SDE) into finite-time convergence bounds for stochastic gradient algorithms. Under generic moment and smoothness conditions, the convergence analysis of the stochastic gradient method reduces to a calculus argument verifying a Lyapunov generator inequality for the SDE. The method recovers established rates for classical stochastic gradient methods and provides a way to establish convergence rates for new optimization algorithms through calculus arguments applied to their associated SDEs.
\end{abstract}
\begin{keywords}
Stochastic Approximation, Stochastic Differential Equations
\end{keywords}

\section{Introduction}

Stochastic gradient descent (SGD) methods evolve through discrete noisy updates.
Their
associated continuous-time models can be easier to analyse. In particular,
stochastic differential equations have a well-developed calculus for studying
convergence. We ask when an SDE convergence bound can be transferred to the
stochastic gradient algorithm.

Consider the stochastic gradient recursion
\[
\bm X_{n+1}=\bm X_n+\alpha_n\{\bm b(\bm X_n)+\bm\varepsilon_{n+1}\}.
\]
Here $\bm b(\bm X_n)$ is the mean update, $\alpha_n>0$
is the stepsize, and the error term $\bm\varepsilon_{n+1}$, conditional on the past, has mean zero and covariance
$\sigma(\bm X_n)\sigma(\bm X_n)^\top$. The associated SDE is
\[
d\bm Z_t=\bm b(\bm Z_t)\,dt+\sqrt{\alpha(t)}\sigma(\bm Z_t)\,d\bm B_t,
\]
Let $\bm B_t$ be a standard Brownian motion. The function $\alpha(t)$
expresses the stepsize sequence on the corresponding diffusion-time scale,
defined in Section~\ref{sec:v3-sa}.

We can analyze convergence of the SDE through a Lyapunov argument. Let \(V\) be a nonnegative, twice-differentiable function measuring the error of interest. The generator of the SDE above is a
differential operator $\mathcal L_tV(\bm x)=\bm b(\bm x)\cdot\nabla V(\bm x)
+\frac{\alpha(t)}{2}\operatorname{tr}\!\left(\sigma(\bm x)\sigma(\bm x)^\top\nabla^2V(\bm x)\right)$.
 We call $V$ a Lyapunov function when there are real functions $\lambda$ and $g$ such that
\[
\mathcal L_tV(\bm x)\leq-\lambda(t)V(\bm x)+g(t).
\]
If the SDE starts at state
$\bm X_0$, the Lyapunov inequality gives the following finite-time bound for
the SDE.
\[\textrm{SDE bound:}\qquad
\mathbb E V(\bm Z_t)\leq e^{-\int_0^t\lambda(r)\,dr}\mathbb E V(\bm X_0)
+\int_0^t e^{-\int_s^t\lambda(r)\,dr}g(s)\,ds.
\]

Under the smoothness and moment conditions of Assumption~\ref{assum:v3-primitive-finite-window}, and assuming that the required moments of the stochastic gradient iterates remain bounded, Theorem~\ref{thm:main-rate-transfer} transfers this SDE estimate to the stochastic gradient method:
\[\textrm{SGD bound:}\quad
\mathbb E V(\bm X_n)\leq e^{-\int_0^{t_n}\lambda(r)\,dr}\mathbb E V(\bm X_0)
+\int_0^{t_n} e^{-\int_s^{t_n}\lambda(r)\,dr}\{g(s)+C\alpha(s)\}\,ds.
\]
Here $t_0=0$ and $t_n=\sum_{k=0}^{n-1}\alpha_k$ for $n\geq1$, so
$t_{n+1}-t_n=\alpha_n$.

Once we have established the finite-time bound for the SDE, it
indicates the likely form of the SGD bound. The stochastic gradient method
inherits the SDE contraction and forcing, while the comparison contributes an
additional forcing term: in particular,
$g(s)\;\longmapsto\;g(s)+C\alpha(s)$. Theorem~\ref{thm:main-rate-transfer}
then establishes the finite-time SGD bound.

We use two examples to illustrate what the rate-transfer method provides. For strongly convex stochastic gradient descent, the transferred SDE estimate recovers the classical convergence rates obtained by direct discrete-time analyses such as \citet{moulines2011nonasymptotic}. For stochastic Heavy Ball, based on the momentum dynamics introduced by \citet{polyak1964speeding}, the same argument recovers the familiar geometric transient together with an $O(h)$ neighbourhood. In both cases, the purpose is not to improve the known rates, but to obtain them from the same continuous-time Lyapunov calculation and finite-window transfer argument.

The basic idea presented here also extends to other forms of Lyapunov analysis. More broadly, the benefit of a transfer theorem is that it removes much of the detailed discrete-time analysis from the convergence proof. A new stochastic gradient method can be analysed by identifying an associated SDE and proving an appropriate Lyapunov inequality. Convergence rates can therefore often be established through a comparatively short calculus argument, rather than a bespoke discrete-time proof.

\subsection{Related work.}
\label{sec:related-work}

Our motivation is stochastic gradient descent methods. However, like prior work, we formulate the transfer theory in the broader language of stochastic approximation, of which SGD is a special case.

Continuous time methods have a long history in stochastic approximation. Stochastic approximation is a noisy Euler discretization of its mean dynamics.
\citet{ljung1977analysis} introduced the ODE viewpoint for recursive stochastic algorithms; contemporaneous work of \citet{kushner1977general} developed closely related convergence methods, and \citet{kushnerclark1978} substantially extended and systematized the differential-equation approach.
This programme was developed further by
\citet{metivierpriouret1984}, \citet{benvenisteetal1990}, and
\citet{benaim1996dynamical}; see also the monographs of
\citet{kushneryin2003} and \citet{borkar2008stochastic}. The central
principle is that a simpler continuous model can reveal the stability and
asymptotic behaviour of a stochastic recursion. Transfer principles also appear in other areas. For example, \citet{dai1995positive} shows that, when every trajectory of the corresponding fluid limit model reaches zero and remains there, the associated queueing network is positive Harris recurrent. This provides an analogue of the SDE-to-SGD transfer principle considered here.

Diffusion approximations also appear early in this literature.
\citet{kushnerhuang1979rates,kushnerhuang1981weak} used weak convergence of
appropriately normalized stochastic-approximation errors to diffusion
processes.
Later work developed strong diffusion approximations for more general
recursive stochastic algorithms \citep{pezeshkiheunis1997strong}, while
diffusion methods continue to be used to describe rates and fluctuations; see,
\citet{kushneryin2003}. Here the focus is on asymptotic scaling under which the normalized error converges weakly to a nondegenerate Gaussian
or diffusion limit.
This is distinct from the finite-time convergence bounds studied here, which control the error at each iteration.

A methodological precursor is \citet{borkarmitter1999strong}, who use finite-time diffusion approximations on successive fixed intervals to obtain long-time information for a constant-stepsize stochastic recursion. More recently, SDEs have been used as quantitative continuous-time models of stochastic gradient algorithms. Stochastic modified equations and weak diffusion approximations have been developed for SGD and related methods~\citep{lietal2017sme,li2019sme,huetal2019diffusion}, while continuous-time convergence arguments have also been used to guide corresponding discrete analyses~\citep{orvietolucchi2019}. Uniform-in-time weak approximations are available in strongly convex settings~\citep{feng2020uniform,liwang2023uniform}. Particularly close to our decreasing-step setting, \citet{fontaine2021convergence} establish finite-horizon approximations of SGD by a time-inhomogeneous SDE and derive, in parallel, corresponding continuous- and discrete-time convergence rates. From the numerical-analysis viewpoint, Gaussian stochastic approximation is an Euler--Maruyama discretisation of the associated SDE, whose finite-horizon approximation theory is classical~\citep{kloedenplaten1992}. A central difficulty is that trajectory-error estimates generally deteriorate as the time horizon grows, a familiar issue in the numerical analysis of dynamical systems~\citep{stuart1994numerical}; under additional dissipativity or stability conditions, time-uniform control may sometimes be recovered~\citep{fanggiles2020adaptive}.

The main challenge addressed by this paper is how to retain a useful SDE estimate over arbitrarily long times when the underlying diffusion approximation is only accurate on finite horizons. Strong approximation approaches seek to keep the stochastic recursion and the diffusion close as processes; however, approximation errors accumulate as the time horizon grows. We instead use only the weaker distributional comparison applied to the Lyapunov function. To establish this comparison over finite times, we use stochastic flow theory of \citet{kunita1990stochastic}, in the forms given by \citet{varadhan2006stochastic} and \cite{matsumototaniguchi2016stochastic}. On each fixed diffusion-time epoch, the SDE is restarted from the current iterate and the local comparison error enters the Lyapunov estimate as an additional forcing term. The key point is that the Lyapunov function is now allowed to attenuate the errors introduced on earlier epochs, without requiring a single diffusion approximation to remain accurate over a growing time horizon. Iterating this argument propagates the SDE Lyapunov estimate over arbitrarily many epochs and yields a quantitative bound at every discrete iteration.

\section{The finite window SDE transfer framework}\label{sec:finite-window-framework}

Our motivation is stochastic gradient descent and related methods. We
formulate the transfer theory for stochastic approximation, which includes
SGD and also covers discrete stochastic recursions that need not arise from a
gradient. We use \emph{stochastic approximation} for the general recursion
and retain SGD terminology when discussing optimization applications.

\subsection{\texorpdfstring{Notation}{Notation}}
\label{sec:mathematical-preliminaries}

For a vector-valued map $\bm b:\R^d\to\R^d$, $D\bm b$ denotes its
Jacobian and $D^j\bm b$ its higher derivatives, viewed as multilinear maps.
For $\sigma:\R^d\to\R^{d\times q}$, $\sigma_j$ denotes its $j$th column.
The symbols $C,c,C_T,\ldots$ denote finite positive constants whose values
may change from line to line, with subscripts indicating relevant dependence.

\[
\polybrac{\bm x}:=(1+\norm{\bm x}_2^2)^{1/2}.
\]
We use $\polybrac{\bm x}^p$ for the polynomial weight of order $p$.

\subsection{Stochastic approximation}\label{sec:v3-sa}
\noindent We now make this broader stochastic-approximation setup
precise. Let
$(\Omega,\mathcal F,\{\mathcal F_n\}_{n\geq0},\mathbb P)$ be a
filtered probability space, and let $\bm X_n\in\R^d$ be adapted.
The stepsizes $\alpha_n>0$ are deterministic, and the total increment
$\bm H_{n+1}$ is $\mathcal F_{n+1}$-measurable. Write
\begin{equation}\label{eq:v3-sa-recursion}
\bm X_{n+1}=\bm X_n+\alpha_n\bm H_{n+1},
\qquad \bm H_{n+1}=\bm b(\bm X_n)+\bm\varepsilon_{n+1}.
\end{equation}
Let $\bm B$ be a $q$-dimensional standard Brownian motion. For functions
$\bm b:\R^d\to\R^d$, $A:\R^d\to\R^{d\times d}$, and
$\sigma:\R^d\to\R^{d\times q}$, the noise satisfies
$\E[\bm\varepsilon_{n+1}\mid\mathcal F_n]=0$ and
$\E[\bm\varepsilon_{n+1}\bm\varepsilon_{n+1}^{\top}\mid\mathcal F_n]
=A(\bm X_n)$.
The covariance factorization is $A(\bm x)=\sigma(\bm x)\sigma(\bm x)^\top$.
The drift $\bm b$ and covariance factor $\sigma$ specify the associated
diffusion. The total increment $\bm H_{n+1}$ remains the quantity used
in the moment assumptions below.
We also need a clock relating algorithmic iterations to diffusion time.

We introduce the diffusion times $t_0=0$, and
\[
t_n=\sum_{k=0}^{n-1}\alpha_k,
\]
and refer to $t_n$ as the diffusion time associated with iteration $n$.
Throughout the paper we assume $t_n\to\infty$ as $n\to\infty$, or equivalently $\sum_n \alpha_n = \infty$.
We define the piecewise-constant stepsize:
\[
\begin{aligned}
\alpha(t)&=\alpha_n,\quad t\in[t_n,t_{n+1}),\qquad
\alpha_{\max}:=\sup_{k\geq0}\alpha_k<\infty.
\end{aligned}
\]
Throughout the paper, $T>0$ denotes a fixed diffusion time horizon,
independent of the iteration indices.

We define, for every $m$,
\begin{equation}\label{eq:fixed-window-endpoint}
N_T(m)=\min\{n>m:t_n-t_m\geq T\}.
\end{equation}
Thus $N_T(m)$ is the next iteration after $m$ at which $T$ units of diffusion
time have passed.

By construction, $T\leq t_{N_T(m)}-t_m<T+\alpha_{\max}$.

We now use these quantities to define the associated SDE.

\subsection{The associated SDE}\label{sec:v3-associated-sde}
The associated diffusion evolves on the diffusion time clock defined above.
Its SDE is
\begin{equation}\label{eq:v3-associated-sde}
d\bm Z_t=\bm b(\bm Z_t)\,dt+
\sqrt{\alpha(t)}\,\sigma(\bm Z_t)\,d\bm B_t,
\end{equation}
The factor $\sqrt{\alpha(t)}$ is chosen so that, over
$[t_n,t_{n+1})$, the local drift and covariance scales are
$\alpha_n\bm b(\bm x)$ and $\alpha_n^2A(\bm x)$, matching the first two
conditional moment scales of the stochastic approximation increment.

For any diffusion time $s\geq0$ and state $\bm x$, let
$\bm Z^{s,\bm x}=(\bm Z_t^{s,\bm x}:t\geq s)$ denote the solution to
\eqref{eq:v3-associated-sde} with $\bm Z_s^{s,\bm x}=\bm x$. We refer to
\eqref{eq:v3-associated-sde} with this initial condition as the
\emph{restarted SDE}.
The restart construction is important to the finite window argument. At
iteration $m$, we compare the stochastic approximation over the next fixed
diffusion time window with the continuous model restarted from
$(t_m,\bm X_m)$. At the next block boundary the comparison is restarted again.

Thus the method requires controlled SA--SDE approximation only on fixed
diffusion time windows, rather than over an ever-growing time interval.
We assume that the SDE is well posed for every restart $(s,\bm x)$ and
is non-explosive on finite intervals. The corresponding restarted solutions
define the evolution operators $P_{s,t}$ below, with the usual composition
property.

To describe the evolution of functions under the restarted diffusion, let
$V$ be any function for which the expectation below is finite and define
\[
P_{s,t}V(\bm x)
=
\mathbb{E}\!\left[V\!\left(\bm Z_t^{s,\bm x}\right)\right],
\qquad s\leq t.
\]
Thus $P_{s,t}V(\bm x)$ is the expected value of $V$ at diffusion time
$t$ when the diffusion starts from $\bm x$ at diffusion time $s$.
Later, $V$ will typically be the Lyapunov function used to measure the
error of the stochastic approximation.
The family $\{P_{s,t}\}_{s\leq t}$ is the evolution family associated with
the time-inhomogeneous SDE. It satisfies $P_{s,s}=I$ and

\begin{equation}\label{eq:P-composition}
P_{s,t}=P_{s,u}P_{u,t},\quad s\leq u\leq t.
\end{equation}
We refer to \eqref{eq:P-composition} as the composition property.
For a fixed terminal index $n$, define
\begin{equation}
u_{k,n}(\bm x)
=P_{t_k,t_n}V(\bm x)
=\mathbb{E}\!\left[
V\!\left(\bm Z_{t_n}^{t_k,\bm x}\right)
\right],
\qquad 0\leq k\leq n.
\label{eq:v3-backward-function}
\end{equation}
We refer to $u_{k,n}(\bm x)$ as the backward function.
Thus $u_{k,n}(\bm x)$ is the expected terminal value of $V$ when the
diffusion is started from $\bm x$ at diffusion time $t_k$ and run until
$t_n$. Since $P_{t_n,t_n}=I$, we have $u_{n,n}=V$. By the composition
property,
\begin{equation}
u_{k,n}
=P_{t_k,t_{k+1}}u_{k+1,n},
\qquad 0\leq k<n.
\label{eq:v3-backward-composition}
\end{equation}

For $\phi\in C^2(\mathbb R^d)$, the time-dependent generator of the SDE~\eqref{eq:v3-associated-sde}
is
\[
\mathcal L_t\phi(\bm x)
=\bm b(\bm x)^\top\nabla\phi(\bm x)
+\frac{\alpha(t)}2 A(\bm x):\nabla^2\phi(\bm x),
\]
Here $A:B=\tr(A^\top B)$ denotes the Frobenius inner product.

\subsection{Lyapunov bounds for the associated SDE}

A key feature of our analysis is that we do not need to solve the associated
SDE explicitly. Instead, we seek a non-negative Lyapunov function $V$ for
which the generator satisfies
\[
\mathcal L_tV(\bm x)
\leq
-\lambda(t)V(\bm x)+g(t).
\]
Here, as with standard Lyapunov analysis, $V$ measures the deviation, energy, or error relevant to the application;
$\lambda(t)$ describes the contraction generated by the continuous dynamics;
and $g(t)$ represents the residual forcing. We take $\lambda$ and $g$ to be
deterministic, nonnegative, and locally integrable.

We can analyze the above Lyapunov bound informally as follows. The formal
argument for this is given in Lemma~1 below.
For the restarted diffusion, Itô's formula gives
\[
dV(\bm Z_t^{s,\bm x})
=\mathcal L_tV(\bm Z_t^{s,\bm x})\,dt+dM_t,
\]
where $M$ is a local martingale. Thus, using the Lyapunov condition,
\[
dV(\bm Z_t^{s,\bm x})
\leq
\left[-\lambda(t)V(\bm Z_t^{s,\bm x})+g(t)\right]dt+dM_t.
\]

We also define $R(s,t)$.
\[
R(s,t)=\exp\left\{-\int_s^t\lambda(r)\,dr\right\}.
\]
The scalar factor $R$ satisfies the following composition property.
\begin{equation}\label{eq:R-composition}
R(s,t)=R(s,u)R(u,t),
\qquad s\leq u\leq t.
\end{equation}

Using an integrating-factor and Grönwall argument, we obtain
\begin{equation}\label{eq:sde-finite-time-lyapunov}
P_{s,t}V(\bm x)
\leq R(s,t)V(\bm x)+\int_s^tR(r,t)g(r)\,dr.
\end{equation}
Here $R(s,t)$ is the contraction inherited from the SDE, while the integral
term is the accumulated forcing. This finite time bound is the
continuous input used later by the transfer theorem.
The important point is that everything to this point is relatively standard calculus.

\subsection{The finite-window transfer mechanism}\label{sec:transfer-idea}

The continuous analysis gives the finite-time bound
\eqref{eq:sde-finite-time-lyapunov}. We now explain how this bound is
transferred to the stochastic approximation. The formal result is
Theorem~\ref{thm:main-rate-transfer}; Proposition~\ref{prop:v3-primitive-window-verification}
supplies its finite-window comparison.

\noindent\emph{Local weak comparison.}

The associated SDE is constructed so that, over one
stochastic-approximation step, its drift and covariance scales match those
of the discrete update. Consequently, Taylor expansion of the SA step and
Dynkin expansion of the SDE have the same leading term, and their one-step
weak difference is of order $\alpha_k^2$. The formal finite-window
comparison below is justified by
Proposition~\ref{prop:v3-primitive-window-verification}.

\[
\left|
\mathbb E[u(\bm X_{k+1})\mid\mathcal F_k]
-P_{t_k,t_{k+1}}u(\bm X_k)
\right|
\leq C_T\alpha_k^2.
\]

For a terminal Lyapunov value at iteration $n$, we use the backward functions
defined in~\eqref{eq:v3-backward-function}. Then $u_{n,n}=V$, and the
one-step backward identity~\eqref{eq:v3-backward-composition} gives
$P_{t_k,t_{k+1}}u_{k+1,n}=u_{k,n}$.
Lemma~\ref{lem:v3-polynomial-window-control} supplies the derivative
bounds needed to apply these one-step estimates to $u_{k+1,n}$, and
Proposition~\ref{prop:v3-primitive-window-verification} makes the resulting
finite-window telescoping rigorous.
Here, we recall that
$\alpha(t)=\alpha_k$ for $t\in[t_k,t_{k+1})$, with
$t_{k+1}-t_k=\alpha_k$.
For $m\leq n\leq N_T(m)$, we telescope the resulting one-step bounds and take
expectations to obtain

\begin{equation}\label{eq:transfer-idea-fixed-window-comparison}
\begin{aligned}
\left|
\mathbb E V(\bm X_n)
-\mathbb E\!\left[P_{t_m,t_n}V(\bm X_m)\right]
\right|
&\leq C_T\sum_{k=m}^{n-1}\alpha_k^2=C_T\int_{t_m}^{t_n}\alpha(s)\,ds.
\end{aligned}
\end{equation}
Here $C_T$ may depend on the fixed window length $T$, but is uniform in the
absolute restart time $t_m$.

\noindent\emph{One-window transfer.}

Substituting the finite-time SDE estimate~\eqref{eq:sde-finite-time-lyapunov}
into the preceding inequality gives
\begin{equation}\label{eq:transfer-idea-one-window-bound}
\mathbb E V(\bm X_n)
\leq R(t_m,t_n)\mathbb E V(\bm X_m)
+\int_{t_m}^{t_n}R(s,t_n)g(s)\,ds
+C_T\int_{t_m}^{t_n}\alpha(s)\,ds.
\end{equation}
The first two terms are exactly the contraction and forcing from the SDE
estimate; the final integral is the approximation cost. This is not yet a
long-time result: the comparison is controlled only on a fixed diffusion-time
window of length $T$.

\noindent\emph{Restarting over successive windows.}

The restart mechanism combines successive fixed-window comparisons.
Recall from~\eqref{eq:fixed-window-endpoint} that $N_T(m)$ is the first
iteration after $m$ at which at least $T$ units of diffusion time have
elapsed. Set $m_0=0$ and $m_{j+1}=N_T(m_j)$.

At each boundary, only the SDE comparator is restarted, from
$(t_{m_j},\bm X_{m_j})$; the stochastic approximation itself continues
normally. The crucial point is that the comparison is made directly at the
Lyapunov observable. The terminal quantity $\mathbb E V(\bm X_{m_{j+1}})$
from one window is exactly the quantity needed to start the next one-window
Lyapunov estimate. Since the comparison constants are uniform in the absolute
restart time, the same fixed-window comparison applies again. The composition
rule~\eqref{eq:R-composition} propagates the SDE contraction across the
windows, and approximation errors from earlier windows are multiplied by the
contraction over subsequent windows. Thus no single SDE approximation needs
to remain accurate over a growing time horizon.

Iterating the preceding window estimate~\eqref{eq:transfer-idea-one-window-bound} over these blocks, then appending one final partial window, gives the
following bound for every $m_J<n\leq m_{J+1}$.

\[
\mathbb E V(\bm X_n)\leq R(0,t_n)\mathbb E V(\bm X_0)
+\int_0^{t_n}R(s,t_n)g(s)\,ds
+C_T\sum_{j=0}^{J}R(t_{m_{j+1}\wedge n},t_n)
\int_{t_{m_j}}^{t_{m_{j+1}\wedge n}}\alpha(s)\,ds.
\]
This is the desired all-iteration estimate: the SDE contraction and forcing
are propagated from time zero, while each local $O(\alpha_k^2)$ weak error
appears as an $\alpha(t)$-scale cost attenuated by the contraction over later
windows. The formal theorem below replaces this schematic calculation by the
precise weighted assumptions needed on an unbounded state space.

\noindent\emph{Why the weak comparison is enough.}

As discussed in Section~\ref{sec:related-work}, diffusion approximations
have also been developed through stronger process-level comparisons. Such
approximations provide more information, but require correspondingly stronger
control. For rate transfer here, it is enough to compare the Lyapunov
observable $V$ weakly over each fixed window. This is exactly the quantity
that closes under the restart mechanism.

Thus a fixed-window weak comparison suffices to propagate the SDE Lyapunov
estimate over all iterations. Under the bounded-contraction condition in the
main theorem, the transported approximation costs can be written as the
additional forcing $C\alpha(s)$, giving the replacement
$g(s)\mapsto g(s)+C\alpha(s)$. The precise statement is
Theorem~\ref{thm:main-rate-transfer}; its proof is in
Appendix~\ref{sec:proof-main-rate-transfer}.

\section{Rate transfer under primitive assumptions}\label{sec:v3-transfer}
Section~\ref{sec:v3-first-order-standing-assumptions} states the
primitive regularity and increment conditions used for the finite-window
comparison. Section~\ref{sec:main-rate-transfer-theorem} gives the main
rate-transfer theorem; its all-iteration conclusion also requires the global
moment stability condition $M_p:=\sup_{n\geq0}\mathbb E\polybrac{\bm X_n}^p<\infty$.
Section~\ref{sec:sde-stability-transfer}
gives a sufficient condition for this stability directly from a
Foster--Lyapunov inequality for the associated SDE.

\subsection{Primitive assumptions}\label{sec:v3-first-order-standing-assumptions}

\begin{assumption}
\label{assum:v3-primitive-finite-window}
For the main primitive result, let
$\bm b\in C^3(\mathbb R^d;\mathbb R^d)$ and
$\sigma\in C^3(\mathbb R^d;\mathbb R^{d\times q})$. Here
$\|\cdot\|_F$ denotes the Frobenius norm. Assume that, for all
$\bm x,\bm y\in\mathbb R^d$,
\begin{equation}
 \|\bm b(\bm x)-\bm b(\bm y)\|\leq L_b\|\bm x-\bm y\|,\qquad
 \|\sigma(\bm x)-\sigma(\bm y)\|_F\leq L_\sigma\|\bm x-\bm y\|.
\label{eq:v3-primitive-global-lipschitz}
\end{equation}
Since the coefficients are $C^1$, this condition implies
\[
 \sup_{\bm x}\|D\bm b(\bm x)\|<\infty,
 \qquad
 \sup_{\bm x}\sum_j\|D\sigma_j(\bm x)\|^2<\infty.
\]
Assume also, for some $\kappa\geq0$,

\begin{equation}
 \|D^2\bm b(\bm x)\|+\sum_j\|D^2\sigma_j(\bm x)\|
 \leq C \polybrac{\bm x}^{\kappa}.
\label{eq:v3-primitive-second-derivative-growth}
\end{equation}
Let $V\in C^2(\mathbb R^d)$ with $V\geq0$ and suppose, for some $v\geq0$, that
\begin{equation}
 \|\nabla V(\bm x)\|+\|\nabla^2V(\bm x)\|
 \leq C \polybrac{\bm x}^{v}.
\label{eq:v3-primitive-lyapunov-growth}
\end{equation}

For some $p\geq v+\kappa+2$, we assume\footnote{More generally, it is
enough for the right-hand side to have polynomial growth of any order no
larger than $p$; we use the self-weighted form below for simplicity.}
\begin{equation}
 \mathbb{E}[\|\bm H_{k+1}\|^p\mid\mathcal F_k]
 \leq M \polybrac{\bm X_k}^{p}.
\label{eq:v3-primitive-oracle-p-moment}
\end{equation}
\end{assumption}

The smoothness and moment conditions in
Assumption~\ref{assum:v3-primitive-finite-window} are comparable to
conditions commonly used in the classical stochastic-approximation
literature; see, for example, \citet{kushneryin2003} and
\citet{benvenisteetal1990}.

These are the primitive conditions used to obtain the finite-window
comparison. They do not include the global moment bound $M_p<\infty$, which
enters only in the all-iteration conclusion of the theorem. Section~
\ref{sec:sde-stability-transfer} gives an SDE-based sufficient condition for
this bound.

\subsection{Main rate-transfer theorem}\label{sec:main-rate-transfer-theorem}

The following is the main result of the paper.

\begin{theorem}[Rate Transfer Theorem]
\label{thm:main-rate-transfer}
Under Assumption~\ref{assum:v3-primitive-finite-window}, for each fixed
$T>0$ and indices $m\leq n\leq N_T(m)$ satisfying
$\mathbb{E}\polybrac{\bm X_m}^{p}<\infty$,
\begin{equation}
\left|\mathbb{E}V(\bm X_n)
-\mathbb{E}\!\left[P_{t_m,t_n}V(\bm X_m)\right]\right|
\leq C_T\mathbb{E}\polybrac{\bm X_m}^{p}
\int_{t_m}^{t_n}\alpha(s)\,ds.
\label{eq:main-rate-transfer-window-comparison}
\end{equation}
If, in addition,
$M_p:=\sup_{k\geq0}\mathbb{E}\polybrac{\bm X_k}^{p}<\infty$ (or equivalently
\(\sup_{k\geq0}\mathbb E\|\bm X_k\|^p<\infty\)) and the
associated SDE satisfies, for deterministic nonnegative locally
integrable functions $\lambda$ and $g$ with
$0\leq\lambda(t)\leq\bar\lambda<\infty$,
\[
\mathcal L_tV(\bm x)\leq-\lambda(t)V(\bm x)+g(t).
\]
With
$R(s,t)=\exp\{-\int_s^t\lambda(r)\,dr\}$, for every $n\geq0$,
\begin{equation}
\mathbb{E}V(\bm X_n)
\leq R(0,t_n)\mathbb{E}V(\bm X_0)
+\int_0^{t_n}R(s,t_n)\{g(s)+C_T\alpha(s)\}\,ds,
\label{eq:v3-first-order-transfer-bound}
\end{equation}
where $C_T<\infty$ is independent of $n$ and may depend on the fixed-window
and moment constants, including $M_p$.
\end{theorem}

The proof combines the primitive fixed-window comparison with the
restart argument; it is given in Appendix~\ref{sec:proof-main-rate-transfer}.
The additional stability hypothesis $M_p<\infty$ is addressed in
Section~\ref{sec:sde-stability-transfer}.

\subsection{Stability from the associated SDE}\label{sec:sde-stability-transfer}

In addition to the primitive finite-window assumptions and the SDE Lyapunov
inequality, Theorem~\ref{thm:main-rate-transfer} requires the global moment
bound
\[
M_p:=\sup_{n\geq0}\mathbb E\polybrac{\bm X_n}^{p}<\infty.
\]
This is a stability property of the stochastic approximation. We now show how
it can be established using a Foster--Lyapunov inequality for the associated
SDE.

As with the Transfer Theorem, we follow the same principle here: we establish a Foster--Lyapunov
inequality for the associated SDE using generator calculus, then transfer
this stability to the stochastic-approximation iterates.
The following
proposition formalizes this step.

\begin{proposition}[Foster--Lyapunov Stability Transfer]
\label{prop:foster-lyapunov-stability-transfer}
Suppose the stochastic approximation and associated SDE satisfy
Assumption~\ref{assum:v3-primitive-finite-window} with moment exponent
$p\geq2$. Let $\Phi\in C^2(\mathbb R^d)$ satisfy $\Phi\geq1$ and, for a
constant $C_\Phi<\infty$,
\[
\polybrac{\bm x}^p\leq C_\Phi\Phi(\bm x),\qquad
\|\nabla\Phi(\bm x)\|\leq C_\Phi\polybrac{\bm x}^{p-1},\qquad
\|\nabla^2\Phi(\bm x)\|\leq C_\Phi\polybrac{\bm x}^{p-2}.
\]
Assume also that $\mathbb E\Phi(\bm X_0)<\infty$ and that, for some
$\gamma>0$, $K\geq0$, and finite index $N_0$, the following bound holds
uniformly for $t\geq t_{N_0}$ and $\bm x\in\mathbb R^d$:
\[
\mathcal L_t\Phi(\bm x)\leq-\gamma\Phi(\bm x)+K.
\]
There is a
sufficiently small $\bar\alpha>0$, depending only on the constants in these
assumptions, such that if there is a finite index $N\geq N_0$ for which
$\alpha_n<\bar\alpha$ for every $n\geq N$, then
\[
\sup_{n\geq0}\mathbb E\Phi(\bm X_n)<\infty,
\qquad
M_p:=\sup_{n\geq0}\mathbb E\polybrac{\bm X_n}^p<\infty.
\]
\end{proposition}

The stability function \(\Phi\) need not coincide with the rate Lyapunov function \(V\) in Theorem~\(\ref{thm:main-rate-transfer}\).

\begin{remark}[Power lifting trick]
\label{rem:power-lifting-sde-stability}
In applications, the natural SDE Lyapunov function often controls a
lower-order moment than the one required by
Theorem~\ref{thm:main-rate-transfer}. For example, if $W$ is quadratic, so
that $1+W(\bm x)\asymp\polybrac{\bm x}^2$, but a $p$-th moment is required,
we use the power-lifting trick and set
\[
\Phi(\bm x)=(1+W(\bm x))^{p/2}.
\]
More generally, if
$1+W(\bm x)\asymp\polybrac{\bm x}^{p_0}$, we take
$\Phi(\bm x)=(1+W(\bm x))^{p/p_0}$. Under the standing regularity
assumptions, the generator chain rule shows that a Foster--Lyapunov bound
for $W$ gives a corresponding bound for $\Phi$ once the stepsizes are
sufficiently small. Since $\Phi(\bm x)\asymp\polybrac{\bm x}^p$,
Proposition~\ref{prop:foster-lyapunov-stability-transfer} then gives the
required $p$-th moment bound. A short verification of this power-lifting step
is given in Lemma~\ref{lem:power-lifting-sde-stability}.
\end{remark}

If $W$ controls a lower-order moment of order $p_0<p$, while
Theorem~\ref{thm:main-rate-transfer} requires order $p$, the power-lifting
trick produces a suitable auxiliary function $\Phi$.
Proposition~\ref{prop:foster-lyapunov-stability-transfer} then establishes
the required uniform $p$-th moment bound:
\[
M_p:=\sup_{n\geq0}\mathbb E\polybrac{\bm X_n}^p<\infty.
\]


\section{Applying the rate-transfer method}\label{sec:examples}
This section gives two applications of the rate-transfer theorem: strongly
convex stochastic gradient descent and stochastic Heavy Ball.

\subsection{How to apply the method}
\label{sec:application-recipe}
Theorem~\ref{thm:main-rate-transfer} reduces an application of the
rate-transfer method to three model-specific steps.
\begin{enumerate}
\item Construct the associated SDE. Identify the conditional drift \(\bm b\) and covariance \(A=\sigma\sigma^\top\) of the stochastic-approximation increment. The associated SDE is
\[
d\bm Z_t=\bm b(\bm Z_t)\,dt
+\sqrt{\alpha(t)}\,\sigma(\bm Z_t)\,d\bm B_t.
\]
Over a time interval of length \(\alpha_n\), its drift and covariance match those of one stochastic-approximation step to leading order.

\item \emph{Establish the SDE Lyapunov bound.} Choose a nonnegative
Lyapunov function $V$ that measures the error or energy of interest. Its
generator is
\[
\mathcal L_tV(\bm x)=
\bm b(\bm x)^\top\nabla V(\bm x)
+\frac{\alpha(t)}{2}A(\bm x):\nabla^2V(\bm x).
\]
Establish, for nonnegative functions $\lambda$ and $g$, the bound
\[
\mathcal L_tV(\bm x)\leq-\lambda(t)V(\bm x)+g(t).
\]
Verify the primitive regularity and moment assumptions of
Theorem~\ref{thm:main-rate-transfer}. If the required uniform moment
$M_p<\infty$ is not immediate, obtain it from a Foster--Lyapunov inequality
for the associated SDE (Proposition~\ref{prop:foster-lyapunov-stability-transfer}),
using the power-lifting trick when a higher moment is needed. The primitive
assumptions and this moment bound give the finite-window comparison needed
for the rate theorem.

\item \emph{Apply the rate-transfer theorem.}
Theorem~\ref{thm:main-rate-transfer} gives
\[
\mathbb E V(\bm X_n)\leq e^{-\int_0^{t_n}\lambda(r)\,dr}
\mathbb E V(\bm X_0)+\int_0^{t_n}e^{-\int_s^{t_n}\lambda(r)\,dr}
\{g(s)+C_T\alpha(s)\}\,ds.
\]
The SDE supplies the contraction through the exponential factors and
the forcing through $g$;
the stochastic approximation contributes the additional
$C_T\alpha(s)$ forcing.
\end{enumerate}

The following examples illustrate these three steps.

\subsection{Strongly convex stochastic gradient descent}
\label{sec:v3-strongly-convex-worked-verification}

To illustrate Theorem~\ref{thm:main-rate-transfer}, we consider
stochastic gradient descent (SGD) for a strongly convex function. Let
$\bm x^\star$ denote the unique minimizer of $f$. Viewed as a stochastic
approximation, the SGD update is:
\[
\bm X_{k+1}=\bm X_k-\alpha_k
\{\nabla f(\bm X_k)+\bm\xi_{k+1}\}.
\]
We assume $f\in C^4(\mathbb R^d)$ is $\mu$-strongly convex with
$L$-Lipschitz gradient, and that there exist $\kappa\geq0$ and $C<\infty$
such that $\|D^3f(\bm x)\|\leq C\polybrac{\bm x}^{\kappa}$ for
$\bm x\in\mathbb R^d$. Set $p:=\kappa+3$. The noise satisfies
$\mathbb E[\bm\xi_{k+1}\mid\mathcal F_k]=0$ and
$\mathbb E[\bm\xi_{k+1}\bm\xi_{k+1}^{\top}\mid\mathcal F_k]=\Sigma$, together
with
$\mathbb E[\|\bm\xi_{k+1}\|^p\mid\mathcal F_k]\leq\nu_p$.
These assumptions are standard in nonasymptotic analyses of strongly convex
SGD; see, for example, \citet{moulines2011nonasymptotic}. We also assume
$\mathbb E\polybrac{\bm X_0}^{p}<\infty$ and nonincreasing stepsizes with
$\alpha_n\to0$. Put
\[
\bm b(\bm x)=-\nabla f(\bm x),\qquad
\sigma=\Sigma^{1/2},\qquad
V(\bm x)=\|\bm x-\bm x^\star\|^2.
\]

This gives the SDE
\[
d\bm Z_t=-\nabla f(\bm Z_t)\,dt
+\sqrt{\alpha(t)}\,\Sigma^{1/2}\,d\bm B_t.
\]

\noindent\emph{Generator estimate.} The generator calculation is
\[
\begin{aligned}
\mathcal L_tV(\bm x)
&=-2(\bm x-\bm x^\star)^\top\nabla f(\bm x)
+\alpha(t)\operatorname{tr}(\Sigma)\\
&\leq-2\mu V(\bm x)+\operatorname{tr}(\Sigma)\alpha(t).
\end{aligned}
\]
Strong convexity and $\nabla f(\bm x^\star)=0$ give the inequality. Thus
$\lambda=2\mu$ and $g(t)=\operatorname{tr}(\Sigma)\alpha(t)$.

\noindent\emph{Theorem assumptions.} Since $f\in C^4$ and $\nabla f$ is
$L$-Lipschitz, $\bm b=-\nabla f$ is $C^3$ and globally Lipschitz;
$\sigma$ is constant, and $D^2\bm b=-D^3f$ has polynomial growth of
order $\kappa$. For quadratic $V$, the derivative-growth bound holds with
$v=1$. Writing
$\bm H_{k+1}=-\nabla f(\bm X_k)-\bm\xi_{k+1}$, the Lipschitz-gradient
bound and the conditional noise moment give

\[
\begin{aligned}
\mathbb E[\|\bm H_{k+1}\|^p\mid\mathcal F_k]
&\leq C_p\left(\|\nabla f(\bm X_k)\|^p
+\mathbb E[\|\bm\xi_{k+1}\|^p\mid\mathcal F_k]\right)\\
&\leq C\polybrac{\bm X_k}^{p}.
\end{aligned}
\]
Thus the primitive oracle moment condition holds. Since
$p=\kappa+3=v+\kappa+2$, Assumption~
\ref{assum:v3-primitive-finite-window} is verified.

Since $\bm x^\star$ is fixed, $1+V(\bm x)\asymp\polybrac{\bm x}^{2}$.
The quadratic function $V$ satisfies the derivative-growth hypotheses of
Lemma~\ref{lem:power-lifting-sde-stability} with $p_0=2$. Moreover,
nonincreasing stepsizes give $\alpha(t)\leq\alpha_0$, so the generator
estimate yields
\[
\mathcal L_tV(\bm x)\leq-2\mu V(\bm x)
+\operatorname{tr}(\Sigma)\alpha_0,
\]
a lower-order Foster--Lyapunov inequality. Set
$\Phi(\bm x)=(1+V(\bm x))^{p/2}$. Lemma~
\ref{lem:power-lifting-sde-stability}, applied with $W=V$ and $p_0=2$,
gives the $p$-order SDE drift inequality when $\alpha(t)$ is sufficiently
small. Since $\alpha_n\to0$, choose $N$ so that this smallness and the
stepsize threshold in Proposition~\ref{prop:foster-lyapunov-stability-transfer}
hold on the tail. The comparison $\Phi(\bm x)\asymp\polybrac{\bm x}^{p}$
shows that the assumed initial moment implies
$\mathbb E\Phi(\bm X_0)<\infty$. Proposition~
\ref{prop:foster-lyapunov-stability-transfer}, whose primitive assumptions
were verified above, then yields the required global moment bound
\[
M_p:=\sup_{k\geq0}\mathbb E\polybrac{\bm X_k}^{p}<\infty.
\]

\noindent\emph{Transfer.} Theorem~\ref{thm:main-rate-transfer} gives

\[
{
\mathbb E\|\bm X_n-\bm x^\star\|^2
\leq e^{-2\mu t_n}\mathbb E\|\bm X_0-\bm x^\star\|^2
+C\int_0^{t_n}e^{-2\mu(t_n-s)}\alpha(s)\,ds.}
\]
Here $C$ absorbs the SDE noise forcing and the finite-window comparison
constant.

\noindent\emph{Rate calculation for power-law stepsizes.}
For $\alpha_k=a/(k+n_0)^\vartheta$ with $0<\vartheta<1$, apply
Corollary~\ref{cor:stepsize-rate-inheritance} with $\lambda=2\mu$ and
$\delta=1$. It gives

\[
\mathbb E\|\bm X_n-\bm x^\star\|^2
=O\!\left(e^{-c(n+n_0)^{1-\vartheta}}
+(n+n_0)^{-\vartheta}\right)
=O(n^{-\vartheta}),
\qquad c=\frac{2\mu a}{1-\vartheta}.
\]
At $\vartheta=1$, the schedule corollary gives the usual three harmonic
regimes according as $2\mu a$ is greater than, equal to, or less than $1$.

\emph{Discussion.} The rate agrees with classical strongly convex SGD
analyses such as Moulines and Bach~\citep{moulines2011nonasymptotic};
Fontaine et al.~\citep{fontaine2021convergence} give a related continuous-time
treatment. Here it follows by transferring the SDE Lyapunov calculation.

\subsection{Stochastic Heavy Ball}
\label{sec:v6-heavy-ball}

We consider the stochastic Heavy Ball recursion
\[
\bm X_{n+1}=\bm X_n+h\bm V_n,\qquad
\bm V_{n+1}=\bm V_n-h\{\eta\bm V_n+\nabla f(\bm X_n)
+\bm\zeta_{n+1}\}.
\]
Let $f\in C^4(\mathbb R^d)$ be $\mu$-strongly convex with
$L$-Lipschitz gradient and bounded third derivative, and set
$\eta=\sqrt{\mu}$. Assume that $\bm\zeta_{n+1}$ is conditionally centred
with covariance
$\sigma(\bm X_n,\bm V_n)\sigma(\bm X_n,\bm V_n)^\top$, where
$\sigma:\mathbb R^{2d}\to\mathbb R^{d\times q}$ is $C^3$ with bounded
first and second derivatives and
$\operatorname{tr}(\sigma\sigma^\top)\leq\varsigma^2$. Suppose
\[
\mathbb E[\|\bm\zeta_{n+1}\|^3\mid\mathcal F_n]
\leq C\{1+\|\bm X_n-\bm x^\star\|^3+\|\bm V_n\|^3\},
\]
and take $h$ sufficiently small.

Write $\bm e=\bm x-\bm x^\star$ and define
\[
W(\bm x,\bm v)
=f(\bm x)-f(\bm x^\star)
+\tfrac12\|\bm v\|^2+\tfrac\eta2\bm e^\top\bm v
+\tfrac\mu4\|\bm e\|^2.
\]
Completing the square and using strong convexity and smoothness give
\[
\begin{aligned}
W(\bm x,\bm v)
&=f(\bm x)-f(\bm x^\star)
+\tfrac12\|\bm v+\tfrac\eta2\bm e\|^2
+\tfrac\mu8\|\bm e\|^2\\
&\geq\tfrac\mu2\|\bm e\|^2+\tfrac14\|\bm v\|^2.
\end{aligned}
\]
Consequently,
\[
1+W(\bm x,\bm v)\asymp\polybrac{(\bm e,\bm v)}^2,\qquad
\|\nabla W(\bm x,\bm v)\|\leq C\polybrac{(\bm e,\bm v)},\qquad
\|D^2W(\bm x,\bm v)\|\leq C.
\]
Assume also the initial moment condition
$\mathbb E(1+W(\bm X_0,\bm V_0))^{3/2}<\infty$.

With $\alpha(t)=h$, this gives the associated SDE
\[
d\bm X_t=\bm V_t\,dt,\qquad
d\bm V_t=-\{\eta\bm V_t+\nabla f(\bm X_t)\}\,dt
+\sqrt{\alpha(t)}\,\sigma(\bm X_t,\bm V_t)\,d\bm B_t.
\]

\noindent\emph{Generator estimate.}
For
$\bm b_{\rm HB}(\bm x,\bm v)
=(\bm v,-\eta\bm v-\nabla f(\bm x))$,
\[
\nabla W(\bm x,\bm v)^\top\bm b_{\rm HB}(\bm x,\bm v)
=-\tfrac\eta2\|\bm v\|^2-\tfrac\eta2\bm e^\top\nabla f(\bm x).
\]
Since $\eta^2=\mu$, the generator calculation gives
\[
\mathcal L_tW+\tfrac\eta2W
\leq-\tfrac\eta8\{\|\bm v\|^2+\|\bm v-\eta\bm e\|^2\}
+\tfrac{\alpha(t)}2\operatorname{tr}(\sigma\sigma^\top).
\]
Thus
\[
\mathcal L_tW
\leq-\lambda_{\rm HB}W+\tfrac12\varsigma^2\alpha(t),
\qquad
\lambda_{\rm HB}=\frac{\sqrt{\mu}}2.
\]

\noindent\emph{Theorem assumptions.}
Set $\bm Y_n=(\bm X_n-\bm x^\star,\bm V_n)$. In these coordinates, the
drift and diffusion factor are
\[
\widetilde{\bm b}_{\rm HB}(\bm e,\bm v)
=(\bm v,-\eta\bm v-\nabla f(\bm x^\star+\bm e)),\qquad
\bm\sigma_{\rm HB}(\bm e,\bm v)
=\begin{pmatrix}
\bm 0_{d\times q}\\
\sigma(\bm x^\star+\bm e,\bm v)
\end{pmatrix}.
\]
The drift is $C^3$ and globally Lipschitz. Bounded $D^3f$ makes its
second derivative bounded, so $\kappa=0$. The diffusion factor is $C^3$,
globally Lipschitz, and has bounded second derivatives. The bounds on
$\nabla W$ and $D^2W$ above give the Lyapunov growth exponent $v=1$.

Let $\bm\varepsilon_{n+1}=(\bm 0,-\bm\zeta_{n+1})$; it is conditionally
centred with covariance $\bm\sigma_{\rm HB}(\bm Y_n)
\bm\sigma_{\rm HB}(\bm Y_n)^\top$. Set
$\bm H_{n+1}=\widetilde{\bm b}_{\rm HB}(\bm Y_n)+\bm\varepsilon_{n+1}$,
so $\bm Y_{n+1}=\bm Y_n+h\bm H_{n+1}$. Linear growth of the drift and
the conditional third-moment bound give
\[
\begin{aligned}
\mathbb E[\|\bm H_{n+1}\|^3\mid\mathcal F_n]
&\leq C\left(\|\widetilde{\bm b}_{\rm HB}(\bm Y_n)\|^3
+\mathbb E[\|\bm\zeta_{n+1}\|^3\mid\mathcal F_n]\right)\\
&\leq C\polybrac{\bm Y_n}^3.
\end{aligned}
\]
Thus $v=1$, $\kappa=0$, and $p=3=v+\kappa+2$, so
Assumption~\ref{assum:v3-primitive-finite-window} holds.

The lower bound on $W$ gives $1+W\asymp\polybrac{\bm Y}^2$, and the
generator estimate gives
$\mathcal L_tW\leq-\lambda_{\rm HB}W+\varsigma^2h_0/2$ uniformly for
$0<h\leq h_0$. The derivative bounds above meet the hypotheses of
Lemma~\ref{lem:power-lifting-sde-stability} with $p_0=2$ and $p=3$.
Set $\Phi=(1+W)^{3/2}$. For sufficiently small $h_0$, that lemma gives
$\mathcal L_t\Phi\leq-c\Phi+C$ uniformly over $0<h\leq h_0$. The initial
moment assumption gives $\mathbb E\Phi(\bm Y_0)<\infty$. Since the primitive
conditions have already been verified, choose $h_0$ below the
stability-transfer threshold and apply
Proposition~\ref{prop:foster-lyapunov-stability-transfer} with $N=0$.
It follows that
\[
\sup_{n\geq0}\mathbb E(1+W(\bm X_n,\bm V_n))^{3/2}<\infty,
\qquad
M_3:=\sup_{n\geq0}\mathbb E\polybrac{\bm Y_n}^3<\infty.
\]

\noindent\emph{Transfer.}
For constant $h$, we have $\alpha(t)=h$, $t_n=nh$,
$\lambda(t)=\lambda_{\rm HB}$, and $g(t)=\varsigma^2h/2$. Theorem~
\ref{thm:main-rate-transfer} gives
\[
\begin{aligned}
\mathbb E W(\bm X_n,\bm V_n)
&\leq e^{-\lambda_{\rm HB}hn}\mathbb E W(\bm X_0,\bm V_0)\\
&\quad+\int_0^{hn}e^{-\lambda_{\rm HB}(hn-s)}
\{\varsigma^2h/2+C_Th\}\,ds\\
&\leq e^{-\lambda_{\rm HB}hn}\mathbb E W(\bm X_0,\bm V_0)+Ch.
\end{aligned}
\]
The constant $C$ is independent of $n$ and can be chosen uniformly over
$0<h\leq h_0$, since the primitive constants and the bound $M_3$ are
uniform over this range. The same energy bounds give
\[
\mathbb E[f(\bm X_n)-f(\bm x^\star)]
\leq\mathbb E W(\bm X_n,\bm V_n),\qquad
\mathbb E\|\bm X_n-\bm x^\star\|^2
\leq\frac2\mu\mathbb E W(\bm X_n,\bm V_n).
\]
Thus the transient is geometric in iteration count, and persistent noise
leaves an $O(h)$ neighbourhood.

\section{Conclusion and extensions}\label{sec:conclusion}

We have shown how a finite-time convergence estimate for an associated
SDE yields a convergence estimate for a stochastic approximation. The method
uses weak comparison only on fixed diffusion-time windows. Restarting the
continuous comparator carries its contraction across successive windows, so
no growing-horizon SDE approximation is needed. The finite-window error
enters the transferred estimate as an additional forcing term.

The two applications do not exhaust the scope of the theorem. The
same proof template can be applied, once its moment and regularity conditions
are verified, e.g. to monotone Robbins--Monro schemes, contractive
fixed-point recursions, stable linear stochastic approximation, and other
optimization recursions with a suitable finite-dimensional SDE Lyapunov
function. There are a host of other applications beyond these.  We omit these further examples to keep the paper focused on the
transfer mechanism.

Other directions require extensions of the theorem rather than only
new applications. These include more general Lyapunov dissipation inequalities,
higher-order stochastic modified equations, path space and hitting time
quantities, Markovian noise, adaptive clocks, and two-timescale recursions.
These questions require additional arguments and are outside the
first-order rate transfer result studied here. The main proof steps are the same however. Again, the focus of this paper is on the transfer mechanism itself, since it contains the core ideas, and we leave these extensions to future work.\\

\noindent \textbf{Acknowledgment.} AI tools were used to assist in the development and checking of mathematical arguments in this paper. The author takes full responsibility for the correctness, originality and presentation of the work.

\bibliographystyle{plainnat}
\bibliography{mybib}
\appendix
\section{Finite-window comparison}\label{sec:appendix-finite-window-comparison}
This appendix establishes the primitive finite-window comparison from the restarted-SDE estimates and the one-step weak comparison.

\subsection{\texorpdfstring{Additional notation and conventions}{Additional notation and conventions}}
\label{sec:appendix-notation}
We retain the notation introduced in
Section~\ref{sec:mathematical-preliminaries} and record here several
additional conventions used in the proofs.

More generally, if $A:(\mathbb R^d)^k\to\mathbb R^m$ is a
$k$-linear map, we write $A[\bm v_1,\ldots,\bm v_k]$ for its evaluation at
$(\bm v_1,\ldots,\bm v_k)$.

For nonnegative quantities $f$ and $g$, we write $f\asymp_p g$ if there
exist constants $0<c_p\leq C_p<\infty$, depending only on $p$, such that
$c_pg\leq f\leq C_pg$.
For \(p\geq0\), the polynomial weight satisfies
\[
\polybrac{\bm x}^p
\asymp_p
1+\norm{\bm x}_2^p.
\]
We repeatedly use
\[
\polybrac{\bm x}^r\polybrac{\bm x}^s
=
\polybrac{\bm x}^{r+s},
\qquad r,s\geq0,
\]
\nopagebreak[4]
and
\nopagebreak[4]
\begin{equation}\label{eq:polybrac-translation}
\polybrac{\bm x+\bm y}^r
\leq
C_r\left\{
\polybrac{\bm x}^r+\norm{\bm y}_2^r
\right\},
\qquad r\geq0.
\end{equation}
Since \(\polybrac{\bm x}\geq1\), if \(0\leq r\leq s\) then
\(\polybrac{\bm x}^r\leq\polybrac{\bm x}^s\).

Correspondingly, $f=O_T(g)$ means that
$|f|\leq C_Tg$ with this same convention.

The symbols $\gamma,\gamma_1,\gamma_2,\ldots$ denote finite nonnegative
exponents whose values are local to the statement or argument in which they
are introduced. Their reuse elsewhere in the paper does not imply that they
have the same value.

Here and below, an empty product is understood to equal $1$.

\subsection{First-order local-to-window approximation}
\label{sec:first-order-local-to-window}

The transfer theorem, Theorem~\ref{thm:finite-window-sde-rate-transfer}, takes two components: the Lyapunov bound \eqref{eq:finite-window-sde-generator} and the finite window approximation \eqref{eq:finite-window-sde-approximation}, and transfers them to a finite time bound for the stochastic approximation.

The Lyapunov bound is a generator inequality for the associated SDE. When an appropriate Lyapunov function is available, it can typically be checked directly by standard calculus from the drift and diffusion coefficients. In contrast, the finite window approximation is a more abstract condition describing the relationship between the stochastic approximation and the SDE over a fixed diffusion time window, and, at first glance, it may not appear as immediately checkable from the model specification.
Here we show how the finite window approximation can be verified from basic regularity and moment conditions on the stochastic approximation and SDE. The precise conditions are given below.
Throughout this subsection we retain the stochastic approximation and
associated-SDE setup of Sections~\ref{sec:v3-sa} and
\ref{sec:v3-associated-sde}, together with the polynomial weight notation of
Section~\ref{sec:mathematical-preliminaries}.

\subsubsection{Primitive finite-window verification}
\label{sec:primitive-finite-window-proof}

\begin{remark}[On the global Lipschitz assumption]
\label{rem:v3-global-lipschitz}
The global Lipschitz condition, together with
$\alpha_{\max}<\infty$, also ensures that the associated SDE is well posed
for every restart. Indeed, global Lipschitzness gives linear growth, and the
time-inhomogeneous diffusion coefficient
$\sqrt{\alpha(t)}\,\sigma(\bm x)$ is uniformly Lipschitz in $\bm x$ and has
linear growth. The standard existence and uniqueness theorem for SDEs
therefore gives a unique nonexplosive strong solution on every finite
interval; see \citet[Theorem~5.2.1]{oksendal2003stochastic}.

The global Lipschitz condition is a convenient sufficient condition rather
than a minimal one. It gives uniformly bounded linear coefficients in the
first variational equation and places the later stochastic-flow estimates in
their standard setting. We do not pursue weaker local conditions here.
Likewise, the linear growth used in the one-step comparison below is
convenient rather than essential: polynomially growing coefficients can be
handled under nonexplosion and sufficiently strong fixed-window polynomial
moment bounds.
\end{remark}

\subsubsection{Primitive finite-window comparison}
\label{sec:first-order-transfer-rates}

The body theorem gives the finite-window comparison directly from the
primitive assumptions. This appendix proves that comparison and then applies
the abstract restart result to obtain the global conclusion. The proof of the
primitive comparison is given in Section~\ref{sec:v3-first-order-window-proof}.

The proof is given at the end of
Section~\ref{sec:v3-first-order-window-proof}. This proposition supplies the
finite-window component of Theorem~\ref{thm:main-rate-transfer}; its global
conclusion follows by applying the abstract SDE transfer theorem under the
stated Lyapunov bound.

{The proof of this result is also given in
Section~\ref{sec:v3-first-order-window-proof}.}
\subsubsection{Proof of the finite window approximation}
\label{sec:v3-first-order-window-proof}

The objective of this subsection is to prove
Proposition~\ref{prop:v3-primitive-window-verification}, which establishes the
finite-window approximation \eqref{eq:finite-window-sde-approximation} required
by Theorem~\ref{thm:finite-window-sde-rate-transfer} under the standing
assumptions.

The proof proceeds through a sequence of results with distinct roles. The
Taylor lemma~\ref{lem:window-two-bound-taylor} gives the
stochastic-approximation one-step remainder for
\[
u_{k,n}(\bm x)=P_{t_k,t_n}V(\bm x)
=\mathbb E\!\left[V\!\left(\bm Z_{t_n}^{t_k,\bm x}\right)\right],
\]
provided that the backward function has polynomial gradient and Hessian
bounds. The target bounds are
\[
\|\nabla P_{s,t}V(\bm x)\|
\leq C_S\polybrac{\bm x}^{v},
\qquad
\|\nabla^2P_{s,t}V(\bm x)\|
\leq C_S\polybrac{\bm x}^{v+\kappa}.
\]
Thus the Taylor lemma will be applied with
$\gamma_1=v$ and $\gamma_2=v+\kappa$.

These derivative bounds require sensitivity of the restarted SDE to its
initial state. The stochastic-flow derivatives $J$ and $K$ satisfy linear
stochastic variational equations, and
Lemma~\ref{lem:window-linear-energy} is the general estimate used to control
those equations. The state moments are separated into their own lemma because
they have three distinct roles. First, the second-variation forcing contains
$D^2\bm b(\bm Z_r)$ and $D^2\sigma_j(\bm Z_r)$, whose growth contributes the
factor $\polybrac{\bm Z_r}^{\kappa}$. Second, differentiating the semigroup
introduces $\nabla V(\bm Z_t)$ and $\nabla^2V(\bm Z_t)$, whose growth
contributes $\polybrac{\bm Z_t}^{v}$.
The state moments also control the displacement and polynomial
weights in Lemma~\ref{lem:v3-sde-one-step-expansion}.

The state-moment lemma supplies the SDE moment input, and the
spatial-variation lemma gives the $J$ and $K$ moments and hence the bounds on
$\nabla P_{s,t}V$ and $\nabla^2P_{s,t}V$.
Lemma~\ref{lem:v3-fixed-window-dynkin} justifies Dynkin's formula for
functions with polynomially growing derivatives. The state moments and
Dynkin formula give Lemma~\ref{lem:v3-sde-one-step-expansion}, while the
Taylor consequence \eqref{eq:window-taylor-conditional-p-moment-bound}
provides the stochastic-approximation estimate under the oracle moment
assumption.
Proposition~\ref{prop:v3-primitive-window-verification} applies both
one-step approximations to the backward function. Comparing their common
first-order expression gives an $O(\alpha_k^2)$ defect, and summing these
conditional defects over the diffusion-time window gives the required
finite-window comparison.

The stochastic-flow argument is carried out in the standard setting in which
$\bm b$ and $\sigma$ are globally Lipschitz and $C^3$, with polynomial growth
allowed for their second derivatives. The extra derivative gives a $C^2$
solution flow, while the quantitative estimates below use only first- and
second-derivative bounds. More general local or one-sided
conditions are possible, but concern the regularity theory of the associated
SDE rather than the finite-window comparison mechanism considered here. We do
not pursue such extensions.

The stochastic-flow references used below are Kunita's standard treatment,
\citet[Section~4.6]{kunita1990stochastic}, and Varadhan's parameter
differentiation argument, \citet[Chapter~3]{varadhan2006stochastic}.


The $p$-moment consequence
\eqref{eq:window-taylor-conditional-p-moment-bound} supplies the
stochastic-approximation one-step estimate once the first two derivatives
of the backward function are controlled. We now develop the corresponding
SDE bounds. The spatial derivatives of
the restarted SDE satisfy linear stochastic variation equations,
and the following energy estimate is the tool used to control them.

\begin{lemma}[An inhomogeneous linear energy estimate]
\label{lem:window-linear-energy}
Consider an adapted $\mathbb R^d$-valued process $\bm Y$ satisfying

\begin{equation}
\label{eq:window-linear-energy-sde}
 d\bm Y_t=(A_t\bm Y_t+\bm f_t)\,dt
 +\sum_j(\Sigma_{j,t}\bm Y_t+\bm g_{j,t})\,dB_t^j,
\end{equation}
where $A_t,\Sigma_{j,t}\in\mathbb R^{d\times d}$ and
$\bm f_t,\bm g_{j,t}\in\mathbb R^d$ are progressively measurable
coefficient processes, and $B^j$ are real-valued Brownian motions.
For matrices, $\|\cdot\|$ denotes the Euclidean operator norm.
Suppose, almost surely,
\begin{equation}
\label{eq:window-linear-energy-coefficient-bounds}
 \langle \bm y,A_t\bm y\rangle\leq L\|\bm y\|^2
 \quad\text{for }\bm y\in\mathbb R^d,\qquad
 \sum_j\|\Sigma_{j,t}\|^2\leq L_1^2.
\end{equation}

Then, for every $p\geq2$ and every fixed $S>0$, whenever $t-s\leq S$,
\begin{equation}
\label{eq:window-linear-energy-estimate}
 \mathbb{E}\sup_{s\leq r\leq t}\|\bm Y_r\|^p
 \leq C^Y_{p,S}\left\{
 \mathbb{E}\|\bm Y_s\|^p+
 \int_s^t\mathbb{E}\left[\|\bm f_r\|^p+
             \left(\sum_j\|\bm g_{j,r}\|^2\right)^{p/2}\right]dr
 \right\}.
\end{equation}

Here $C^Y_{p,S}$ may depend on the fixed constants $L$ and $L_1$, but not on
an operator-norm bound for $A_t$.
\end{lemma}
\Beginproof{}
The argument here is a relatively standard stochastic calculus argument.
To prove the lemma, we first introduce a state-localizing sequence and apply
It\^o's formula to the stopped process. We then bound the finite-variation and
quadratic-variation terms; the latter shows that the stopped stochastic
integral is a true martingale. We use the Burkholder--Davis--Gundy inequality
and Gronwall's inequality, and finally remove the stopping by Fatou's lemma.

We define
\[
G_r:=\Big(\sum_j\|\bm g_{j,r}\|^2\Big)^{1/2},
\qquad
\bm Z_{j,r}:=\Sigma_{j,r}\bm Y_r+\bm g_{j,r}.
\]
Thus $\bm f_r$ is the part of the drift independent of $\bm Y_r$, while
$\bm g_{j,r}$ are the corresponding parts of the diffusion coefficients, and
$G_r$ records their combined Euclidean size. We can assume that

$\mathbb E\|\bm Y_s\|^p<\infty$, $\int_s^t\mathbb E\|\bm f_r\|^p\,dr<\infty$,
and $\int_s^t\mathbb E G_r^p\,dr<\infty$. If any of these quantities were
infinite, the required bound, \eqref{eq:window-linear-energy-estimate}, would
be immediate. Define, for $n\geq1$,
\[
\tau_n:=\inf\{u\in[s,t]:\|\bm Y_u\|\geq n\}\wedge t.
\]

We use this state-localizing sequence to justify the martingale estimates
below. Since $\bm Y$ has continuous paths on the compact interval $[s,t]$,
$\sup_{s\leq r\leq t}\|\bm Y_r\|<\infty$ almost surely, so $\tau_n=t$ for all
sufficiently large $n$ almost surely, and hence $\tau_n\uparrow t$ almost
surely.

For
$u\in[s,t]$, It\^o's
formula gives

\begin{align}
\|\bm Y_{u\wedge\tau_n}\|^p
={}&\|\bm Y_s\|^p
\notag\\
&+p\int_s^{u\wedge\tau_n}
 \|\bm Y_r\|^{p-2}
 \langle \bm Y_r,A_r\bm Y_r+\bm f_r\rangle\,dr
\notag\\
&+\frac p2\int_s^{u\wedge\tau_n}
 \|\bm Y_r\|^{p-2}\sum_j\|\bm Z_{j,r}\|^2\,dr
\notag\\
&+\frac{p(p-2)}2
 \int_s^{u\wedge\tau_n}
 \|\bm Y_r\|^{p-4}\sum_j\langle \bm Y_r,\bm Z_{j,r}\rangle^2\,dr
\notag\\
&+p\sum_j\int_s^{u\wedge\tau_n}
 \|\bm Y_r\|^{p-2}
 \langle \bm Y_r,\bm Z_{j,r}\rangle\,dB_r^j,
\label{eq:window-linear-energy-ito}
\end{align}
For $u\in[s,t]$, write $M_u$ for the same stochastic integral with upper
limit $u$. Thus the final term above is $M_{u\wedge\tau_n}$, and $M$ is the
local-martingale term produced by It\^o's formula.
For $p=2$, the term multiplied by $p-2$ is interpreted as zero.

We now bound the finite variation and the martingale terms above.
The
finite variation integrand in \eqref{eq:window-linear-energy-ito} satisfies

\begin{align*}
& p\|\bm Y_r\|^{p-2}\langle \bm Y_r,A_r\bm Y_r+\bm f_r\rangle
 +\frac p2\|\bm Y_r\|^{p-2}\sum_j\|\bm Z_{j,r}\|^2
 +\frac{p(p-2)}2\|\bm Y_r\|^{p-4}
   \sum_j\langle \bm Y_r,\bm Z_{j,r}\rangle^2
\notag\\
&\leq pL\|\bm Y_r\|^p+p\|\bm Y_r\|^{p-1}\|\bm f_r\|
 +\frac{p(p-1)}2\|\bm Y_r\|^{p-2}\sum_j\|\bm Z_{j,r}\|^2
\\
&\leq pL\|\bm Y_r\|^p+(p-1)\|\bm Y_r\|^p+\|\bm f_r\|^p
 +\frac{p(p-1)}2\|\bm Y_r\|^{p-2}
   \left(2L_1^2\|\bm Y_r\|^2+2G_r^2\right)
\\
&\leq C_1
 \left\{\|\bm Y_r\|^p+\|\bm f_r\|^p+G_r^p\right\}.
\end{align*}
Above, we apply the coefficient bounds in
\eqref{eq:window-linear-energy-coefficient-bounds} and the Cauchy--Schwarz inequality. We also use the inequality:

\begin{equation}
\label{eq:window-linear-energy-Z-bound}
\sum_j\|\bm Z_{j,r}\|^2
\le 2\sum_j\|\Sigma_{j,r}\bm Y_r\|^2+2\sum_j\|\bm g_{j,r}\|^2
\le 2L_1^2\|\bm Y_r\|^2+2G_r^2.
\end{equation}
In the last line, we use Young's inequality as follows:
\[
p\|\bm Y_r\|^{p-1}\|\bm f_r\|
\leq(p-1)\|\bm Y_r\|^p+\|\bm f_r\|^p,
\quad\text{and}\quad
\|\bm Y_r\|^{p-2}G_r^2
\leq
\frac{p-2}{p}\|\bm Y_r\|^p
+\frac{2}{p}G_r^p.
\]
This gives the above finite-variation bound.

The resulting finite-variation estimate following from
\eqref{eq:window-linear-energy-ito} is
\begin{align}
\|\bm Y_{u\wedge\tau_n}\|^p
\leq{}&\|\bm Y_s\|^p
+C_1\int_s^{u\wedge\tau_n}
 \{\|\bm Y_r\|^p+\|\bm f_r\|^p+G_r^p\}\,dr
+M_{u\wedge\tau_n},
\label{eq:window-linear-energy-finite-variation}
\end{align}

where $C_1$ depends only on $p,L,L_1$.

It remains to justify the martingale term. We first record a
quadratic-variation bound that will be used twice: first to show that the
stopped local martingale is square-integrable, and later in the BDG estimate.

Define
\[
Y_n^*(r):=\sup_{s\leq v\leq r}\|\bm Y_{v\wedge\tau_n}\|^p.
\]
By continuity and the definition of $\tau_n$, $Y_n^*(r)\leq
n^p\vee\|\bm Y_s\|^p$, and hence $\mathbb E Y_n^*(r)<\infty$.
For $r\in[s,t]$, the pathwise quadratic variation satisfies
\begin{align}
\langle M\rangle_{r\wedge\tau_n}
&=p^2\int_s^{r\wedge\tau_n}
 \|\bm Y_q\|^{2p-4}
 \sum_j\langle\bm Y_q,\bm Z_{j,q}\rangle^2\,dq
\notag\\
&\leq p^2\int_s^{r\wedge\tau_n}
 \|\bm Y_q\|^{2p-2}
 \sum_j\|\bm Z_{j,q}\|^2\,dq
\notag\\
&\leq C_2\int_s^{r\wedge\tau_n}
 \left\{\|\bm Y_q\|^{2p}
 +\|\bm Y_q\|^{2p-2}G_q^2\right\}\,dq
\notag\\
&\leq C_2
 \bigl(Y_n^*(r)\wedge n^p\bigr)
 \int_s^{r\wedge\tau_n}
 \left\{\|\bm Y_q\|^p
 +\|\bm Y_q\|^{p-2}G_q^2\right\}\,dq.
\label{eq:window-linear-energy-qv}
\end{align}

The first inequality is Cauchy--Schwarz and the second uses
\eqref{eq:window-linear-energy-Z-bound}. For $q<\tau_n$,
$\|\bm Y_q\|^p\leq n^p$, while
$\|\bm Y_q\|^p\leq Y_n^*(r)$, which gives the final line. The
hitting-time endpoint is irrelevant to the time integral.

For the localization, we now use the crude bounds supplied by $\tau_n$:

\begin{align*}
\mathbb E\langle M\rangle_{t\wedge\tau_n}
&\leq C_2
\mathbb E\left[
\bigl(Y_n^*(t)\wedge n^p\bigr)
\int_s^{t\wedge\tau_n}
 \left\{\|\bm Y_r\|^p
 +\|\bm Y_r\|^{p-2}G_r^2\right\}\,dr
\right]\\
&\leq C_2
 \left\{Sn^{2p}
 +n^{2p-2}\mathbb E\int_s^tG_r^2\,dr\right\}\\
&\leq C_2
 \left\{Sn^{2p}
 +n^{2p-2}S^{1-2/p}
\mathbb E\left[
\left(\int_s^tG_r^p\,dr\right)^{2/p}
\right]\right\}\\
&\leq C_2
 \left\{Sn^{2p}
 +n^{2p-2}S^{1-2/p}
\left(\int_s^t\mathbb E G_r^p\,dr\right)^{2/p}\right\}\\
&<\infty.
\end{align*}
The second line uses the localization bounds
$Y_n^*(t)\wedge n^p\leq n^p$ and $\|\bm Y_r\|\leq n$ before $\tau_n$.
The next inequality uses Hölder's inequality in time together with
$t-s\leq S$, and the following one uses Jensen's inequality in probability
and Tonelli's theorem. The finite-right-hand-side assumption gives
$\int_s^t\mathbb E G_r^p\,dr<\infty$. Therefore
$M_{\cdot\wedge\tau_n}$ is a square-integrable martingale, hence a true
martingale.

By the Burkholder--Davis--Gundy (BDG) inequality and Young's inequality,

the complete estimate is
\begin{align}
\mathbb E\sup_{s\leq u\leq r}|M_{u\wedge\tau_n}|
\leq{}& C_p\mathbb E\langle M\rangle_{r\wedge\tau_n}^{1/2}\notag\\
\leq{}& C_3\mathbb E\left[
\bigl(Y_n^*(r)\wedge n^p\bigr)^{1/2}
\left(\int_s^{r\wedge\tau_n}
 \left\{\|\bm Y_q\|^p
 +\|\bm Y_q\|^{p-2}G_q^2\right\}\,dq
\right)^{1/2}
\right]\notag\\
\leq{}& C_3\mathbb E\left[
\bigl(Y_n^*(r)\bigr)^{1/2}
\left(\int_s^{r\wedge\tau_n}
 \left\{\|\bm Y_q\|^p
 +\|\bm Y_q\|^{p-2}G_q^2\right\}\,dq
\right)^{1/2}
\right]\notag\\
\leq{}&\frac12\mathbb E Y_n^*(r)
+C_4\int_s^r
\left\{\mathbb E Y_n^*(q)+\mathbb E G_q^p\right\}\,dq.
\label{eq:window-linear-energy-bdg}
\end{align}
The first line is the Burkholder--Davis--Gundy inequality, and the next uses
\eqref{eq:window-linear-energy-qv}. We then use
$Y_n^*(r)\wedge n^p\leq Y_n^*(r)$. The final line follows from Young's
inequality in the form
$C_3a^{1/2}b^{1/2}\leq \frac12a+C_4b$ for a new constant $C_4$, together
with the previously displayed bound on $\|\bm Y_q\|^{p-2}G_q^2$.

Finally, set

\[
\Phi_n(r):=\mathbb E Y_n^*(r)
=\mathbb E\sup_{s\leq u\leq r}\|\bm Y_{u\wedge\tau_n}\|^p.
\]
For $q\leq r\wedge\tau_n$, $\|\bm Y_q\|^p\leq Y_n^*(q)$, while the
remaining integrands are nonnegative. Hence, taking the supremum over
$u\leq r$ in \eqref{eq:window-linear-energy-finite-variation}, extending the
integrals to $[s,r]$, and taking expectations gives
\begin{align*}
\Phi_n(r)
&\leq \mathbb E\|\bm Y_s\|^p
 +C_1\int_s^r\mathbb E[\|\bm f_q\|^p+G_q^p]\,dq
 +C_1\int_s^r\Phi_n(q)\,dq\\
&\qquad +\mathbb E\sup_{s\leq u\leq r}|M_{u\wedge\tau_n}|\\
&\leq{} \mathbb E\|\bm Y_s\|^p
 +C_1\int_s^r\mathbb E[\|\bm f_q\|^p+G_q^p]\,dq
 +C_1\int_s^r\Phi_n(q)\,dq
 \\
 & \qquad +\frac12\Phi_n(r)
 +C_4\int_s^r
 \{\Phi_n(q)+\mathbb E G_q^p\}\,dq\\
&\leq{} \frac12\Phi_n(r)
 +C_5\mathbb E\|\bm Y_s\|^p
 +C_5\int_s^r
 \left\{\mathbb E\|\bm f_q\|^p+\mathbb E G_q^p+\Phi_n(q)\right\}\,dq\\
&\leq{} C_6\left\{
 \mathbb E\|\bm Y_s\|^p
 +\int_s^r\mathbb E[\|\bm f_q\|^p+G_q^p]\,dq
 \right\}
 +C_6\int_s^r\Phi_n(q)\,dq.
\end{align*}

The second inequality uses \eqref{eq:window-linear-energy-bdg} together with
$\Phi_n(q)=\mathbb E Y_n^*(q)$. The next line collects the $C_1$ and $C_4$
terms into a new constant $C_5$. Absorbing $\frac12\Phi_n(r)$ into the
left-hand side gives the Gronwall inequality with a new constant $C_6$.
Since
$r-s\leq S$, Gronwall's inequality gives, uniformly in $n$,
\[
\Phi_n(t)
\leq C^Y_{p,S}\left\{
\mathbb E\|\bm Y_s\|^p+
\int_s^t\mathbb E[\|\bm f_q\|^p+G_q^p]\,dq
\right\}.
\]
The window-length assumption is used here to absorb the exponential Gronwall
factor into the constant $C^Y_{p,S}$.

The constants $C_1,\ldots,C_6$ depend only on $p,L,L_1$ and are independent
of $n,s,t$. After Gronwall on an interval of length at most $S$, these
constants and the exponential Gronwall factor are absorbed into
$C^Y_{p,S}$.

It remains to remove the stopping. Since $\tau_n\uparrow t$ almost surely
and $\bm Y$ is continuous,

\[
\sup_{s\leq r\leq t}\|\bm Y_{r\wedge\tau_n}\|^p
\longrightarrow
\sup_{s\leq r\leq t}\|\bm Y_r\|^p
\quad\text{almost surely}.
\]
Fatou's lemma and the uniform bound above yield
\begin{align*}
\mathbb E\sup_{s\leq r\leq t}\|\bm Y_r\|^p
&\leq\liminf_{n\to\infty}
\mathbb E\sup_{s\leq r\leq t}\|\bm Y_{r\wedge\tau_n}\|^p\\
&\leq C^Y_{p,S}\left\{
\mathbb E\|\bm Y_s\|^p+
\int_s^t\mathbb E[\|\bm f_r\|^p+G_r^p]\,dr
\right\}.
\end{align*}

The preceding display is exactly
\eqref{eq:window-linear-energy-estimate}. Since $A_t$ enters only through
the first inequality in
\eqref{eq:window-linear-energy-coefficient-bounds}, no operator-norm bound
on $A_t$ is required.
\Endproof

We use that continuous adapted processes, and continuous functions of them,
are progressively measurable; multiplication by deterministic Borel
functions such as $\sqrt{\alpha(t)}$ preserves this property.

The linear energy estimate will first be applied to the restarted state
itself. Uniform control over the entire fixed window is needed for three
reasons. Differentiating $P_{s,t}V$ introduces polynomial factors involving
$\polybrac{\bm Z_t^{s,\bm x}}$ for arbitrary $t\leq s+S$. The second
spatial-variation equation contains the random factor
$\polybrac{\bm Z_r}^{\kappa}$ at varying integration times $r$, and the
later applications of H\"older's inequality require higher polynomial
moment orders. The path-supremum estimate in the following lemma supplies
all of these fixed-time controls at once.

The same linear estimate is then applied to the first and second spatial
variations. Their moment bounds, together with the state moments, yield
\[
\|\nabla P_{s,t}V(\bm x)\|
\leq C_S\polybrac{\bm x}^{v},
\qquad
\|\nabla^2P_{s,t}V(\bm x)\|
\leq C_S\polybrac{\bm x}^{v+\kappa}.
\]
Since $\kappa\geq0$, these are precisely the derivative-growth hypotheses
needed in Lemma~\ref{lem:window-two-bound-taylor}, with
$\gamma_1=v$ and $\gamma_2=v+\kappa$.

\begin{lemma}[Fixed-window moments of the restarted SDE]
\label{lem:v3-restarted-state-moments}

Suppose $\bm b\in C^1(\mathbb R^d;\mathbb R^d)$ and
$\sigma\in C^1(\mathbb R^d;\mathbb R^{d\times q})$ satisfy the global
Lipschitz condition
\eqref{eq:v3-primitive-global-lipschitz}, and suppose
$\alpha_{\max}<\infty$. Then, for every fixed $S>0$ and every finite $p>0$,
\begin{equation}
\mathbb E\sup_{s\leq r\leq s+S}
\polybrac{\bm Z_r^{s,\bm x}}^p
\leq
C^Z_{p,S}\polybrac{\bm x}^p.
\label{eq:v3-restarted-state-moments}
\end{equation}
Here $C^Z_{p,S}$ may depend on $(p,S)$ and the fixed coefficient data, but
not on $s$ or $\bm x$.
\end{lemma}

\Beginproof{}
We rewrite the restarted SDE in the affine-linear form required by
Lemma~\ref{lem:window-linear-energy} and verify its coefficient and forcing
conditions. The linear energy estimate then gives
the required moment bound for $p\geq2$, while Jensen's inequality gives the
remaining range $0<p<2$.

The well-posedness observation in
Remark~\ref{rem:v3-global-lipschitz} uses only global Lipschitzness and
boundedness of $\alpha$, so it applies under the present assumptions. The
$C^1$ regularity is used below to express the coefficients in the
affine-linear form required by Lemma~\ref{lem:window-linear-energy}.

\emph{Reduction to the linear form.}
Write $\bm Z_r=\bm Z_r^{s,\bm x}$, which we recall satisfies
\[
d\bm Z_r=\bm b(\bm Z_r)\,dr
+\sqrt{\alpha(r)}\sum_j\sigma_j(\bm Z_r)\,dB_r^j,
\qquad \bm Z_s=\bm x.
\]
We now express the drift and each diffusion column in affine-linear form.
For each $r$ and $j$, set $A_r:=\int_0^1D\bm b(\theta\bm Z_r)\,d\theta$,
$\Sigma_{j,r}:=\sqrt{\alpha(r)}\int_0^1D\sigma_j(\theta\bm Z_r)\,d\theta$,
and $\bm g_{j,r}:=\sqrt{\alpha(r)}\,\sigma_j(\bm 0)$. By the fundamental
theorem of calculus,

\[
\begin{aligned}
\bm b(\bm Z_r)
&=\int_0^1\frac{d}{d\theta}\bm b(\theta\bm Z_r)\,d\theta+\bm b(\bm 0)\\
&=\int_0^1D\bm b(\theta\bm Z_r)\bm Z_r\,d\theta+\bm b(\bm 0)
=A_r\bm Z_r+\bm b(\bm 0).
\end{aligned}
\]
Similarly, for each diffusion column $\sigma_j$,

\[
\sqrt{\alpha(r)}\,\sigma_j(\bm Z_r)=\Sigma_{j,r}\bm Z_r+\bm g_{j,r}.
\]
Thus the SDE becomes

\[
d\bm Z_r=(A_r\bm Z_r+\bm b(\bm 0))\,dr
+\sum_j(\Sigma_{j,r}\bm Z_r+\bm g_{j,r})\,dB_r^j,
\qquad \bm Z_s=\bm x.
\]
This is the form of Lemma~\ref{lem:window-linear-energy}, with
$\bm Y_r=\bm Z_r$, $\bm f_r=\bm b(\bm 0)$, and $A_r$, $\Sigma_{j,r}$, and
$\bm g_{j,r}$ as defined above.

The integral maps defining $A_r$ and $\Sigma_{j,r}$ are continuous in
$\bm Z_r$, so $A_r$, $\Sigma_{j,r}$, $\bm g_{j,r}$, and $\bm f_r$ are
progressively measurable by the observation following
Lemma~\ref{lem:window-linear-energy}.

\emph{Coefficient bounds.}

The global Lipschitz and $C^1$ assumptions give finite constants
\[
K_b:=\sup_{\bm z}\|D\bm b(\bm z)\|,
\qquad
K_\sigma^2:=
\sup_{\bm z}\sum_j\|D\sigma_j(\bm z)\|^2.
\]
Hence
\[
\langle\bm y,A_r\bm y\rangle
\leq
\|A_r\|\|\bm y\|^2
\leq
K_b\|\bm y\|^2,
\]
so Lemma~\ref{lem:window-linear-energy} may take $L=K_b$. Jensen's
inequality in $\theta$ gives
\[
\begin{aligned}
\sum_j\|\Sigma_{j,r}\|^2
&=
\alpha(r)\sum_j
\left\|
\int_0^1D\sigma_j(\theta\bm Z_r)\,d\theta
\right\|^2\\
&\leq
\alpha(r)\sum_j\int_0^1
\|D\sigma_j(\theta\bm Z_r)\|^2\,d\theta\\
&\leq
\alpha_{\max}K_\sigma^2,
\end{aligned}
\]
so one may take $L_1^2=\alpha_{\max}K_\sigma^2$.

For $\bm f_r$, we have $\|\bm f_r\|^p=\|\bm b(\bm 0)\|^p$.
For $\bm g_{j,r}$,
\[
\left(\sum_j\|\bm g_{j,r}\|^2\right)^{p/2}
=
\alpha(r)^{p/2}\|\sigma(\bm 0)\|_F^p
\leq
\alpha_{\max}^{p/2}\|\sigma(\bm 0)\|_F^p.
\]

Together with the progressive measurability noted above, these estimates
verify all the hypotheses of Lemma~\ref{lem:window-linear-energy}, with
$L=K_b$ and $L_1^2=\alpha_{\max}K_\sigma^2$.

\emph{Application of the linear estimate.}
For $p\geq2$, Lemma~\ref{lem:window-linear-energy} gives
\[
\begin{aligned}
\mathbb E\sup_{s\leq r\leq s+S}\|\bm Z_r\|^p
&\leq
C^Y_{p,S}
\left\{
\|\bm x\|^p
+\int_s^{s+S}
\left[
\|\bm b(\bm 0)\|^p
+
\alpha(r)^{p/2}\|\sigma(\bm 0)\|_F^p
\right]dr
\right\}\\
\end{aligned}
\]
Since $\alpha(r)\leq\alpha_{\max}$ and the interval has length $S$, the
integral is bounded by a finite constant independent of $\bm x$. Absorbing
this constant into a new constant $C^Z_{p,S}$ gives
\[
\mathbb E\sup_{s\leq r\leq s+S}\|\bm Z_r\|^p
\leq C^Z_{p,S}\bigl(1+\|\bm x\|^p\bigr).
\]
Since $\polybrac{\bm z}^p\asymp_p1+\|\bm z\|^p$, this implies
\[
\mathbb E\sup_{s\leq r\leq s+S}
\polybrac{\bm Z_r^{s,\bm x}}^p
\leq
C^Z_{p,S}\polybrac{\bm x}^p,
\qquad p\geq2.
\]
For $0<p<2$, Jensen's inequality applied to the $p=2$ estimate gives
\[
\begin{aligned}
\mathbb E\sup_{s\leq r\leq s+S}
\polybrac{\bm Z_r^{s,\bm x}}^p
&=
\mathbb E\left[
\sup_{s\leq r\leq s+S}
\polybrac{\bm Z_r^{s,\bm x}}^2
\right]^{p/2}\\
&\leq
\left(
\mathbb E\sup_{s\leq r\leq s+S}
\polybrac{\bm Z_r^{s,\bm x}}^2
\right)^{p/2}\\
&\leq
\left(C^Z_{2,S}\right)^{p/2}\polybrac{\bm x}^p.
\end{aligned}
\]
Absorbing $\left(C^Z_{2,S}\right)^{p/2}$ into a constant denoted by
$C^Z_{p,S}$ gives the stated bound for $0<p<2$.
This proves the fixed-window state-moment bound for every finite $p>0$, and
completes the proof.

\Endproof

To obtain gradient and Hessian bounds for $P_{s,t}V$, we control the first and
second spatial variations of the restarted SDE flow.
The process $\bm Z_t^{s,\bm x}$ can be differentiated with respect to its initial
condition. The resulting spatial derivatives are described by the classical
theory of stochastic flows, developed systematically by
\citep{kunita1990stochastic}. 
Here we use
\citet[Theorem~4.10.8]{matsumototaniguchi2016stochastic}, applied after
localization to the resulting $C^3$ vector fields.

\begin{lemma}[Polynomial control of semigroup derivatives]
\label{lem:v3-polynomial-window-control}

Under \eqref{eq:v3-primitive-global-lipschitz},
\eqref{eq:v3-primitive-second-derivative-growth}, and
\eqref{eq:v3-primitive-lyapunov-growth} of
Assumption~\ref{assum:v3-primitive-finite-window}, for every fixed
$S>0$ and $0\leq t-s\leq S$, $P_{s,t}V\in C^2(\mathbb R^d)$, with
\begin{equation}
\begin{aligned}
\|\nabla P_{s,t}V(\bm x)\|
&\leq C_S\polybrac{\bm x}^{v},\\
\|\nabla^2P_{s,t}V(\bm x)\|
&\leq C_S\polybrac{\bm x}^{v+\kappa}.
\end{aligned}
\label{eq:v3-polynomial-window-derivative-bounds}
\end{equation}
The constant $C_S$ may depend on the fixed coefficient data and on $S$, but
not on $s,t$, or $\bm x$ subject to $0\leq t-s\leq S$.
\end{lemma}
\Beginproof{}
The proof proceeds by calculating the first and second derivatives of the restarted flow $\bm Z_r^{s,\bm x}$ with respect to its initial state $\bm x$.
To do this  we apply  \citet{matsumototaniguchi2016stochastic}, Theorem~4.10.8. We then apply Lemma~\ref{lem:window-linear-energy} to bound the maximum
moments of the flow. 

\emph{First spatial derivative of the flow.}
Applying Theorem~4.10.8 of \citet{matsumototaniguchi2016stochastic} to
bounded $C^3$ vector fields shows that the Jacobian
$J_r:=D_{\bm x}\bm Z_r^{s,\bm x}$ exists with $J_s=I$ and satisfies the
first-variation equation:

\begin{equation}
\label{eq:v3-first-variation-equation}
dJ_r
=D\bm b(\bm Z_r)J_r\,dr
+\sqrt{\alpha(r)}\sum_jD\sigma_j(\bm Z_r)J_r\,dB_r^j.
\end{equation}

There are two technical points in applying
\citet[Theorem~4.10.8]{matsumototaniguchi2016stochastic}. First, that theorem
is applied to bounded $C^3$ vector fields with bounded derivatives through
order three, whereas our coefficients are $C^3$ and their higher derivatives
need only be locally bounded. For each $R>0$, a smooth cutoff gives globally
bounded $C^3$ extensions $\bm b^{(R)}$ and $\sigma^{(R)}$ that agree with
$\bm b$ and $\sigma$ on the ball $B(\bm 0,R)$ and have bounded derivatives
through order three. Applying the theorem to these modified coefficients gives a
$C^2$ stochastic flow. By pathwise uniqueness, the modified and original
restarted diffusions, and hence their local flow maps, agree up to the exit
time from $B(\bm 0,R)$. Since the original restarted SDE is nonexplosive,
these exit times tend to infinity almost surely as $R\to\infty$, yielding the
local $C^2$ flow for the original equation; this is the standard localization
argument, see also \citet[Remark~3.1]{varadhan2006stochastic}.

Second, $\alpha(r)$ is piecewise constant. On each interval
$[t_k,t_{k+1})$, it equals $\alpha_k$, so the spatial vector fields are
$\bm b$ and $\sqrt{\alpha_k}\,\sigma_j$. The corresponding cutoff vector
fields are bounded $C^3$ fields, so Theorem~4.10.8 applies on that interval.
A fixed diffusion-time window
intersects only finitely many such intervals, and composing the corresponding
$C^2$ solution maps gives a $C^2$ restarted flow throughout the window.
Consequently the first variation satisfies
\eqref{eq:v3-first-variation-equation}. The same
localization-and-composition argument gives the second spatial derivative
and hence the second-variation equation used below. The $C^3$ assumption is
used here to obtain the $C^2$ flow; no global bound on third derivatives is
assumed or used in the quantitative estimates below, which rely only on the
stated bounds on the first and second derivatives.

For a fixed direction $\bm a$, set
$\bm Y_r:=J_r\bm a$, with $\bm Y_s=\bm a$. Then $\bm Y$ satisfies

\[
 d\bm Y_r=A_r\bm Y_r\,dr+\sum_j\Sigma_{j,r}\bm Y_r\,dB_r^j,
\]
where $A_r=D\bm b(\bm Z_r)$, $\Sigma_{j,r}=\sqrt{\alpha(r)}D\sigma_j(\bm Z_r)$.
This is precisely the homogeneous case of
\eqref{eq:window-linear-energy-sde} in
Lemma~\ref{lem:window-linear-energy}, with $\bm f_r=0$ and
$\bm g_{j,r}=0$. We next verify the coefficient bounds required to apply
that lemma.
Differentiating the global Lipschitz bounds in
Assumption~\ref{assum:v3-primitive-finite-window} gives
$\|D\bm b(\bm x)\|\leq L_b$ and
$\|D\sigma_j(\bm x)\|\leq L_\sigma$ for $j=1,\ldots,q$.
By the measurability observation following
Lemma~\ref{lem:window-linear-energy}, the coefficient processes $A_r$ and
$\Sigma_{j,r}$ are progressively measurable. Their coefficient bounds give
\[
\langle\bm y,A_r\bm y\rangle
\leq L_b\|\bm y\|^2,
\qquad
\sum_j\|\Sigma_{j,r}\|^2
=\alpha(r)\sum_{j=1}^q\|D\sigma_j(\bm Z_r)\|^2
\leq\alpha_{\max}qL_\sigma^2.
\]
Together with progressive measurability and $\bm f_r=\bm g_{j,r}=0$, these
bounds verify all hypotheses of Lemma~\ref{lem:window-linear-energy}, with
$L=L_b$ and $L_1^2=\alpha_{\max}qL_\sigma^2$. Thus, by
Lemma~\ref{lem:window-linear-energy}, we have

\begin{equation}
\mathbb{E}\sup_{s\leq r\leq s+S}\|J_r\bm a\|^p
\leq C^Y_{p,S}\|\bm a\|^p,
\qquad p\geq2.
\label{eq:v3-first-variation-sensitivity}
\end{equation}
For the standard basis $\bm e_1,\ldots,\bm e_d$, finite dimensionality gives
$\|J_r\|\leq\|J_r\|_F\leq\sum_{i=1}^d\|J_r\bm e_i\|$. Therefore,
\begin{equation}
\mathbb E\sup_{s\leq r\leq s+S}\|J_r\|^p
\leq
C_{p,d}\sum_{i=1}^d
\mathbb E\sup_{s\leq r\leq s+S}\|J_r\bm e_i\|^p
\leq C^J_{p,S}.
\label{eq:v3-first-variation-operator-moment}
\end{equation}

This proves the required first spatial derivative moment bound. We now use it,
together with the state-moment lemma, to bound the $\bm f_r$ and
$\bm g_{j,r}$ terms in the second spatial derivative equation.

\emph{Second spatial derivative of the flow.}
The second spatial derivative is handled similarly to the first.
Since the restarted flow is $C^2$ in its initial state, for fixed
 directions $\bm a,\bm c$ we may define
 $K_r[\bm a,\bm c]:=D_{\bm x}^2\bm Z_r^{s,\bm x}[\bm a,\bm c]$ with $K_s[\bm a,\bm c]=0$.
 \footnote{The $C^3$ coefficient assumption in
 Assumption~\ref{assum:v3-primitive-finite-window} lets us obtain the $C^2$ flow
 directly from Theorem~4.10.8. With only $C^2$ coefficients, one could
 instead construct the second spatial variation from finite differences of
 the Jacobian equation, at the cost of a longer argument. We impose the
 additional derivative only to streamline the proof.}
 Differentiating the Jacobian equation in direction $\bm c$, as in the proof
 of \citet[Theorem~4.10.8]{matsumototaniguchi2016stochastic}, gives the
second-variation equation. The same two technical points as for the
 first variation apply: localization gives the required $C^2$ dependence on
 the initial state, while the piecewise-constant $\alpha$ is handled interval
 by interval and the resulting $C^2$ flow maps are composed across the
 finitely many boundaries.

\begin{align}
 dK_r[\bm a,\bm c]
 ={}&\left\{
 D\bm b(\bm Z_r)K_r[\bm a,\bm c]
 +D^2\bm b(\bm Z_r)[J_r\bm c,J_r\bm a]
 \right\}dr \notag\\
 &+\sqrt{\alpha(r)}\sum_j\left\{
 D\sigma_j(\bm Z_r)K_r[\bm a,\bm c]
 +D^2\sigma_j(\bm Z_r)[J_r\bm c,J_r\bm a]
 \right\}dB_r^j.
\label{eq:v3-second-variation-equation}
\end{align}

 Now set $\bm Y_r=K_r[\bm a,\bm c]$, with $\bm Y_s=0$, and take
 $A_r=D\bm b(\bm Z_r)$,
 $\Sigma_{j,r}=\sqrt{\alpha(r)}D\sigma_j(\bm Z_r)$,
 $\bm f_r=D^2\bm b(\bm Z_r)[J_r\bm c,J_r\bm a]$, and
 $\bm g_{j,r}=\sqrt{\alpha(r)}D^2\sigma_j(\bm Z_r)
 [J_r\bm c,J_r\bm a]$. The coefficient hypotheses for $A_r$ and
 $\Sigma_{j,r}$ were verified in the first spatial derivative argument. By
 the measurability observation following Lemma~\ref{lem:window-linear-energy}
 and continuity of the second derivatives, $\bm f_r$ and $\bm g_{j,r}$ are
 progressively measurable. Define
 $G_r=\left(\sum_j\|\bm g_{j,r}\|^2\right)^{1/2}$. Now apply
 Lemma~\ref{lem:window-linear-energy}, using $\bm Y_s=0$:
 \begin{equation}
 \mathbb E\sup_{s\leq r\leq s+S}
 \|K_r[\bm a,\bm c]\|^p
 \leq C^Y_{p,S}\int_s^{s+S}
 \mathbb E[\|\bm f_r\|^p+G_r^p] \,dr.
 \label{eq:v3-second-variation-linear-energy}
 \end{equation}
 We now look to bound the $\bm f_r$ and $G_r$ terms in the integrand above.

 By
 \eqref{eq:v3-primitive-second-derivative-growth} in
 Assumption~\ref{assum:v3-primitive-finite-window},
 \[
 \begin{aligned}
 \|\bm f_r\|
 &\leq
 \|D^2\bm b(\bm Z_r)\|\,\|J_r\bm c\|\,\|J_r\bm a\|\\
 &\leq
 C\polybrac{\bm Z_r}^{\kappa}
 \|J_r\|^2\|\bm a\|\,\|\bm c\|,
 \end{aligned}
 \]
 and, using $\left(\sum_j a_j^2\right)^{1/2}\leq\sum_j a_j$,
 \[
 \begin{aligned}
 G_r
 &\leq
 \sqrt{\alpha(r)}\,\|J_r\|^2\|\bm a\|\,\|\bm c\|
 \left(\sum_j\|D^2\sigma_j(\bm Z_r)\|^2\right)^{1/2}\\
 &\leq
 \sqrt{\alpha(r)}\,\|J_r\|^2\|\bm a\|\,\|\bm c\|
 \sum_j\|D^2\sigma_j(\bm Z_r)\|\\
 &\leq
 C\polybrac{\bm Z_r}^{\kappa}
 \|J_r\|^2\|\bm a\|\,\|\bm c\|,
 \end{aligned}
 \]
 The last inequality uses the second-derivative growth condition
 \eqref{eq:v3-primitive-second-derivative-growth} from
 Assumption~\ref{assum:v3-primitive-finite-window}, with
 $\sqrt{\alpha_{\max}}$ absorbed into $C$. Taking $p$th powers in the
 preceding estimates and applying H\"older's inequality gives
 \[
 \mathbb E[\polybrac{\bm Z_r}^{\kappa p}\|J_r\|^{2p}]
 \leq
 \left\{\mathbb E\polybrac{\bm Z_r}^{2\kappa p}\right\}^{1/2}
 \left\{\mathbb E\|J_r\|^{4p}\right\}^{1/2}.
 \]
 Together with the state and Jacobian moment bounds, these estimates and
 H\"older's inequality give
 \begin{equation}
 \mathbb E[\|\bm f_r\|^p+G_r^p]
 \leq C^K_{p,S}\polybrac{\bm x}^{\kappa p}
 \|\bm a\|^p\|\bm c\|^p.
 \label{eq:v3-second-variation-fg-bound}
 \end{equation}
 If $\kappa=0$, the first factor in the Hölder estimate is simply one.
 Substituting \eqref{eq:v3-second-variation-fg-bound} into
 \eqref{eq:v3-second-variation-linear-energy} and absorbing the constant $C^Y_{p,S}$ into $C^K_{p,S}$ gives

 \begin{equation}
 \mathbb E\sup_{s\leq r\leq s+S}
 \|K_r[\bm a,\bm c]\|^p
 \leq C^K_{p,S}\polybrac{\bm x}^{\kappa p}
 \|\bm a\|^p\|\bm c\|^p.
 \label{eq:v3-second-variation-moment-bound}
 \end{equation}
(By finite dimensionality, applying the preceding estimate to the
finitely many coordinate pairs also gives the corresponding moment bound for
the bilinear operator norm of $K_t$.)

\emph{Differentiation of the semigroup.}
It remains to translate the estimates
\eqref{eq:v3-restarted-state-moments},
\eqref{eq:v3-first-variation-operator-moment}, and
\eqref{eq:v3-second-variation-moment-bound}
into the semigroup function $P_{s,t}V$. (For the
stochastic-approximation argument, this is required to analyse the backward
function $u_{k,n}=P_{t_k,t_n}V$ defined in
\eqref{eq:v3-backward-function}.)

\[
\|\nabla P_{s,t}V(\bm x)\|
\leq C_S\polybrac{\bm x}^{v},
\qquad
\|\nabla^2P_{s,t}V(\bm x)\|
\leq C_S\polybrac{\bm x}^{v+\kappa}.
\]
Given our stochastic flow analysis the pathwise derivatives of
$V(\bm Z_t^{s,\bm x})$ follow directly by the chain rule. (The state, Jacobian, and
second-variation moment estimates \eqref{eq:v3-restarted-state-moments},
\eqref{eq:v3-first-variation-operator-moment}, and
\eqref{eq:v3-second-variation-moment-bound}, in addition to providing regularity bounds, imply uniform integrability. This allows us to justify passing these derivatives through
the expectation defining $P_{s,t}V$.)

To display the dependence on the initial state in this part of the proof,
write $J_t^{\bm y}$ and $K_t^{\bm y}$ for the variations defined above
along $\bm Z_t^{s,\bm y}$.

For fixed $s,t$, define
$\Phi(\bm y):=V(\bm Z_t^{s,\bm y})$. The restarted flow is $C^2$ in its
initial state $\bm y$. For directions
$\bm a,\bm c$, the ordinary pathwise chain rule gives

\[
D\Phi(\bm y)[\bm a]
=
\nabla V(\bm Z_t^{s,\bm y})^\top J_t^{\bm y}\bm a.
\]

The second pathwise derivative is

\[
D^2\Phi(\bm y)[\bm a,\bm c]
=
\nabla^2V(\bm Z_t^{s,\bm y})[J_t^{\bm y}\bm c,J_t^{\bm y}\bm a]
+
\nabla V(\bm Z_t^{s,\bm y})^\top K_t^{\bm y}[\bm a,\bm c].
\]
Consequently,
\[
\|D\Phi(\bm y)\|
\leq C\polybrac{\bm Z_t^{s,\bm y}}^v\|J_t^{\bm y}\|, \qquad
\|D^2\Phi(\bm y)\|
\leq C\polybrac{\bm Z_t^{s,\bm y}}^v
\left(\|J_t^{\bm y}\|^2+\|K_t^{\bm y}\|\right).
\]

On every bounded neighbourhood of $\bm x$, the state-moment estimate in
Lemma~\ref{lem:v3-restarted-state-moments}, together with the Jacobian moment
bound \eqref{eq:v3-first-variation-operator-moment} and the second-variation
bound \eqref{eq:v3-second-variation-moment-bound}, gives uniform moments of
$\bm Z$, $J$, and $K$. The state moments control the polynomial factor in the displayed
derivative bounds; the $J$ moments control the first-derivative and
$J^2$ terms, while the $K$ moments control the remaining second-derivative
term. By Assumption~\ref{assum:v3-primitive-finite-window}, specifically
\eqref{eq:v3-primitive-lyapunov-growth}, $\nabla V$ and
$\nabla^2V$ have polynomial growth. For $V$ itself, the fundamental theorem
of calculus gives
\[
\begin{aligned}
|V(\bm z)|
&\leq |V(\bm 0)|
+\|\bm z\|\int_0^1\|\nabla V(\theta\bm z)\|\,d\theta \leq C\polybrac{\bm z}^{v+1}.
\end{aligned}
\]
Thus the state-moment estimate also ensures that
$V(\bm Z_t^{s,\bm y})$ is integrable. H\"older's inequality and the local
moment bounds imply locally uniform $L^2$ bounds for $D\Phi$ and $D^2\Phi$.
These bounds imply uniform integrability on bounded sets of initial
conditions. Since the pathwise derivatives are continuous in $\bm y$, the
standard differentiation-under-expectation argument applies to
$P_{s,t}V(\bm x)=\mathbb E[V(\bm Z_t^{s,\bm x})]$. It yields
$P_{s,t}V\in C^2(\mathbb R^d)$ and the derivative formulas below.

\begin{equation}
\nabla P_{s,t}V(\bm x)^\top\bm a
=
\mathbb E\!\left[
\nabla V(\bm Z_t^{s,\bm x})^\top J_t\bm a
\right].
\label{eq:v3-semigroup-gradient-representation}
\end{equation}
Analogously, the second derivative is

\begin{equation}
\begin{aligned}
\nabla^2P_{s,t}V(\bm x)[\bm a,\bm c]
=
\mathbb E\Big[
&\nabla^2V(\bm Z_t^{s,\bm x})
[J_t\bm c,J_t\bm a]+
\nabla V(\bm Z_t^{s,\bm x})^\top K_t[\bm a,\bm c]
\Big].
\end{aligned}
\label{eq:v3-semigroup-hessian-representation}
\end{equation}

Equations \eqref{eq:v3-semigroup-gradient-representation} and
\eqref{eq:v3-semigroup-hessian-representation} give the required polynomial
bounds. Specifically, from \eqref{eq:v3-semigroup-gradient-representation},
\eqref{eq:v3-primitive-lyapunov-growth}, and the state and Jacobian moment
estimates \eqref{eq:v3-restarted-state-moments} and
\eqref{eq:v3-first-variation-operator-moment}, Cauchy--Schwarz gives
\[
|\nabla P_{s,t}V(\bm x)^\top\bm a|
\leq C
\left\{\mathbb E\polybrac{\bm Z_t^{s,\bm x}}^{2v}\right\}^{1/2}
\left\{\mathbb E\|J_t\|^2\right\}^{1/2}
\|\bm a\|
\leq C_S\polybrac{\bm x}^{v}\|\bm a\|.
\]
From the first term in
\eqref{eq:v3-semigroup-hessian-representation},
\eqref{eq:v3-primitive-lyapunov-growth}, and the state and Jacobian moment
estimates \eqref{eq:v3-restarted-state-moments} and
\eqref{eq:v3-first-variation-operator-moment},
\[
\left|\mathbb E\,\nabla^2V(\bm Z_t^{s,\bm x})
[J_t\bm c,J_t\bm a]\right|
\leq C_S\polybrac{\bm x}^{v}\|\bm a\|\|\bm c\|.
\]
For the second term in \eqref{eq:v3-semigroup-hessian-representation}, the
same growth assumption, together with the state and second-variation moment
estimates \eqref{eq:v3-restarted-state-moments} and
\eqref{eq:v3-second-variation-moment-bound}, gives
\[
\left|\mathbb E\,\nabla V(\bm Z_t^{s,\bm x})^\top
K_t[\bm a,\bm c]\right|
\leq C_S\polybrac{\bm x}^{v+\kappa}\|\bm a\|\|\bm c\|.
\]
Since $\kappa\geq0$ and $\polybrac{\bm x}\geq1$, these estimates give
\[
\|\nabla P_{s,t}V(\bm x)\|
\leq
C_S\polybrac{\bm x}^{v},
\qquad
\|\nabla^2P_{s,t}V(\bm x)\|
\leq
C_S\polybrac{\bm x}^{v+\kappa}.
\]

We have therefore proved the $C^2$ regularity and the gradient and Hessian
bounds stated in Lemma~\ref{lem:v3-polynomial-window-control}. In
particular, $P_{s,t}V$ satisfies the hypotheses of
Lemma~\ref{lem:window-two-bound-taylor} with
$\gamma_1=v$ and $\gamma_2=v+\kappa$.
\Endproof

The preceding lemma gives the derivative bounds needed for the one-step
comparison.
The following lemma records the standard Dynkin formula in the
form needed here; see, for example,
\citet[Theorem~7.4.1]{oksendal2003stochastic}. We include the short argument
to verify its applicability under our polynomial growth bounds.
(The argument is standard and can be skipped on first reading.)

\begin{lemma}[Taylor remainder using both derivative bounds]
\label{lem:window-two-bound-taylor}
Suppose $u\in C^2(\mathbb R^d)$ and, for $0\leq\gamma_1\leq\gamma_2$,

\begin{equation}
 \|\nabla u(\bm x)\|\leq C_1 \polybrac{\bm x}^{\gamma_1},\qquad
 \|\nabla^2u(\bm x)\|\leq C_2 \polybrac{\bm x}^{\gamma_2}.
\label{eq:window-taylor-derivative-growth}
\end{equation}
Set $\gamma_3 = \max\{2,\gamma_1+1\}$. Then, for $\bm z \in \mathbb R^d$,
\begin{equation}
 |u(\bm x+\bm z)-u(\bm x)-\nabla u(\bm x)^\top\bm z|
 \leq C_3\{\polybrac{\bm x}^{\gamma_2}\|\bm z\|^2+\|\bm z\|^{\gamma_3}\}.
\label{eq:window-taylor-deterministic-remainder}
\end{equation}
Consequently, if
$\mathbb{E}[\|\bm H\|^{\gamma_3}\mid\mathcal F]\leq M \polybrac{\bm x}^{\gamma_4}$,
 and $\bm x$ is $\mathcal F$-measurable, then

\begin{equation}
 \left|\mathbb{E}[u(\bm x+h\bm H)\mid\mathcal F]-u(\bm x)
       -h\nabla u(\bm x)^\top\mathbb{E}[\bm H\mid\mathcal F]\right|
 \leq C_4 h^2 \polybrac{\bm x}^{\gamma_5}\, ,
\label{eq:window-taylor-conditional-bound}
\end{equation}
for $0<h\leq\alpha_{\max}$ and
where
\(
\gamma_5
=
\max\{
\gamma_2+ {2\gamma_4}/{\gamma_3},
\gamma_4
\}.
\)

In particular, if $\bm x$ is $\mathcal F$-measurable and, for some
$p\geq\gamma_2+2$,
\[
\mathbb E[\|\bm H\|^p\mid\mathcal F]\leq M_p\polybrac{\bm x}^{p},
\]
then, for $0<h\leq\alpha_{\max}$,
\begin{equation}
\begin{aligned}
\Big|\,\mathbb E[u(\bm x+h\bm H)\mid\mathcal F]
&-u(\bm x)\\
&-h\nabla u(\bm x)^\top\mathbb E[\bm H\mid\mathcal F]\,\Big|
\leq C_5h^2\polybrac{\bm x}^{\gamma_2+2}.
\end{aligned}
\label{eq:window-taylor-conditional-p-moment-bound}
\end{equation}
Here $C_5$ depends only on $C_1,C_2$, the derivative exponents
$\gamma_1,\gamma_2$, $p$, $M_p$, and $\alpha_{\max}$.

\end{lemma}
\Beginproof{}
For $\|\bm z\|\leq1$, Taylor's formula with integral remainder gives
\begin{align}
\left|u(\bm x+\bm z)-u(\bm x)-\nabla u(\bm x)^\top\bm z\right|
&=
\left|\int_0^1
(1-\theta)\,\bm z^\top\nabla^2u(\bm x+\theta\bm z)\bm z
\,d\theta\right|
\notag\\
&\leq
C_2\|\bm z\|^2\int_0^1(1-\theta)
\polybrac{\bm x+\theta\bm z}^{\gamma_2}\,d\theta
\label{eq:window-taylor-integral-remainder} \\
&\leq 2^{\gamma_2}C_2\polybrac{\bm x}^{\gamma_2}
\|\bm z\|^2\int_0^1(1-\theta)\,d\theta\label{eq:window-taylor-small-increment-bound-1} \\
&=2^{\gamma_2-1}C_2\polybrac{\bm x}^{\gamma_2}\|\bm z\|^2.
\notag
\end{align}
In the inequality \eqref{eq:window-taylor-integral-remainder},
we apply the assumed bound \eqref{eq:window-taylor-derivative-growth}. For the inequality \eqref{eq:window-taylor-small-increment-bound-1}, note that the regularized Euclidean bracket satisfies
\(
\polybrac{\bm x+\bm y}\leq\polybrac{\bm x}+\|\bm y\|.
\)
Consequently, we can apply the inequality
\[
\polybrac{\bm x+\theta\bm z}^{\gamma_2}
\leq\left(\polybrac{\bm x}+\theta\|\bm z\|\right)^{\gamma_2}
\leq\left(\polybrac{\bm x}+1\right)^{\gamma_2}
\leq 2^{\gamma_2}\polybrac{\bm x}^{\gamma_2},
\]
which is the required bound when $\|\bm z\|\leq1$.

For $\|\bm z\|>1$, using the Hessian bound globally would introduce the
higher power $\|\bm z\|^{\gamma_2+2}$. Instead, we use the gradient bound along
the segment from $\bm x$ to $\bm x+\bm z$:
\begin{align}
&\left|
u(\bm x+\bm z)-u(\bm x)-\nabla u(\bm x)^\top\bm z
\right|
\notag\\
&\qquad=
\left|
\int_0^1
\nabla u(\bm x+\theta\bm z)^\top\bm z\,d\theta
-\nabla u(\bm x)^\top\bm z
\right|
\label{eq:window-taylor-large-ftc}\\
&\qquad\leq
\|\bm z\|
\int_0^1
\|\nabla u(\bm x+\theta\bm z)\|\,d\theta
+\|\bm z\|\,\|\nabla u(\bm x)\|
\label{eq:window-taylor-large-triangle}\\
&\qquad\leq
C_1\|\bm z\|
\left\{
\int_0^1
\polybrac{\bm x+\theta\bm z}^{\gamma_1}\,d\theta
+\polybrac{\bm x}^{\gamma_1}
\right\}
\label{eq:window-taylor-large-gradient}\\
&\qquad\leq
C_1(1+C_{\gamma_1})\|\bm z\|
\left\{
\polybrac{\bm x}^{\gamma_1}
+\|\bm z\|^{\gamma_1}
\right\}
\label{eq:window-taylor-large-translation}\\
&\qquad\leq
C_1(1+C_{\gamma_1})
\left\{
\polybrac{\bm x}^{\gamma_2}\|\bm z\|^2
+\|\bm z\|^{\gamma_3}
\right\}.
\label{eq:window-taylor-large-final}
\end{align}
In \eqref{eq:window-taylor-large-ftc}, we use the fundamental theorem of
calculus along the segment from $\bm x$ to $\bm x+\bm z$. In
\eqref{eq:window-taylor-large-triangle}, we use the triangle inequality and
the Euclidean inner-product bound. In
\eqref{eq:window-taylor-large-gradient}, we apply the gradient-growth
assumption from \eqref{eq:window-taylor-derivative-growth}. In
\eqref{eq:window-taylor-large-translation}, we use the translation inequality for $\polybrac{\bm x + \bm y}$
\eqref{eq:polybrac-translation} with
$\bm y=\theta\bm z$, together with $0\leq\theta\leq1$. For
\eqref{eq:window-taylor-large-final}, since
$\gamma_1\leq\gamma_2$, $\polybrac{\bm x}\geq1$, and
$\|\bm z\|>1$, we have
$\polybrac{\bm x}^{\gamma_1}\|\bm z\|
\leq\polybrac{\bm x}^{\gamma_2}\|\bm z\|^2$; moreover,
$\|\bm z\|^{\gamma_1+1}\leq\|\bm z\|^{\gamma_3}$ because
$\gamma_3=\max\{2,\gamma_1+1\}$.
Choose
\[
C_3=\max\left\{2^{\gamma_2-1}C_2,\,C_1(1+C_{\gamma_1})\right\}.
\]
Then the two cases give the deterministic Taylor-remainder bound
\eqref{eq:window-taylor-deterministic-remainder}.

\begin{align}
&\left|
\mathbb{E}[u(\bm x+h\bm H)\mid\mathcal F]
-u(\bm x)
-h\nabla u(\bm x)^\top
  \mathbb{E}[\bm H\mid\mathcal F]
\right|
\notag\\
&\qquad\leq
\mathbb{E}\!\left[
\left|
u(\bm x+h\bm H)-u(\bm x)
-h\nabla u(\bm x)^\top\bm H
\right|
\,\middle|\,\mathcal F
\right]
\notag\\
&\qquad\leq
C_3h^2\polybrac{\bm x}^{\gamma_2}
\mathbb{E}[\|\bm H\|^2\mid\mathcal F]
+C_3h^{\gamma_3}
\mathbb{E}[\|\bm H\|^{\gamma_3}\mid\mathcal F]
\label{eq:window-taylor-conditional-remainder}\\
&\qquad\leq
C_3M^{2/\gamma_3}h^2
\polybrac{\bm x}^{\gamma_2+2\gamma_4/\gamma_3}
+C_3Mh^{\gamma_3}
\polybrac{\bm x}^{\gamma_4}
\label{eq:window-taylor-conditional-moments}\\
&\qquad\leq
C_3h^2
\left\{
M^{2/\gamma_3}
\polybrac{\bm x}^{\gamma_2+2\gamma_4/\gamma_3}
+M\alpha_{\max}^{\gamma_3-2}
\polybrac{\bm x}^{\gamma_4}
\right\}
\label{eq:window-taylor-conditional-stepsize}\\
&\qquad\leq
C_4h^2\polybrac{\bm x}^{\gamma_5}
.
\label{eq:window-taylor-conditional-final}
\end{align}
In
\eqref{eq:window-taylor-conditional-remainder}, we apply the deterministic
Taylor-remainder estimate \eqref{eq:window-taylor-deterministic-remainder}
with $\bm z=h\bm H$.
For \eqref{eq:window-taylor-conditional-moments}, the assumed
conditional moment assumption gives the $\gamma_3$-moment term directly. Also
(conditional Jensen) and $\gamma_3\geq2$ give that
$\mathbb E[\|\bm H\|^2\mid\mathcal F]\leq
\left(\mathbb E[\|\bm H\|^{\gamma_3}\mid\mathcal F]\right)^{2/\gamma_3}
\leq M^{2/\gamma_3}\polybrac{\bm x}^{2\gamma_4/\gamma_3}$.
For
\eqref{eq:window-taylor-conditional-stepsize}, since $\gamma_3\geq2$ and
$0<h\leq\alpha_{\max}$, we use
$h^{\gamma_3}\leq\alpha_{\max}^{\gamma_3-2}h^2$. Finally, for
\eqref{eq:window-taylor-conditional-final}, the definition of
$\gamma_5$ and $\polybrac{\bm x}\geq1$ imply that both preceding polynomial
weights are bounded by $\polybrac{\bm x}^{\gamma_5}$, while the constant \(
C_4=C_3\left(1+M^{2/\gamma_3}
+M\alpha_{\max}^{\gamma_3-2}\right).
\) dominates
the two coefficients. This proves
\eqref{eq:window-taylor-conditional-bound}.

For the $p$-moment specialization,
$\gamma_3=\max\{2,\gamma_1+1\}\leq\gamma_2+2\leq p$.
Conditional Jensen's inequality gives
\[
\begin{aligned}
\mathbb E[\|\bm H\|^{\gamma_3}\mid\mathcal F]
&\leq\left(\mathbb E[\|\bm H\|^p\mid\mathcal F]\right)^{\gamma_3/p}\\
&\leq M_p^{\gamma_3/p}\polybrac{\bm x}^{\gamma_3}.
\end{aligned}
\]
Thus \eqref{eq:window-taylor-conditional-bound} applies with
$\gamma_4=\gamma_3$, and its exponent becomes
\[
\gamma_5
=\max\left\{\gamma_2+\frac{2\gamma_4}{\gamma_3},\gamma_4\right\}
=\max\{\gamma_2+2,\gamma_3\}
=\gamma_2+2.
\]
This proves \eqref{eq:window-taylor-conditional-p-moment-bound}.
\Endproof
\begin{lemma}[Dynkin formula on a fixed window]
\label{lem:v3-fixed-window-dynkin}
Under the assumptions of Lemma~\ref{lem:v3-restarted-state-moments}, fix
$S>0$. Let $u\in C^2(\mathbb R^d)$ and suppose that, for some $a,b\geq0$,

\begin{equation}
\|\nabla u(\bm y)\|
\leq C_u\polybrac{\bm y}^{a},
\qquad
\|\nabla^2u(\bm y)\|
\leq C_u\polybrac{\bm y}^{b}.
\label{eq:v3-dynkin-test-function-derivative-growth}
\end{equation}
Then, for $0\leq t-s\leq S$,
\begin{equation}
\begin{aligned}
P_{s,t}u(\bm x)
={}&u(\bm x)
+\int_s^t
\mathbb E\!\left[
\bm b(\bm Z_r^{s,\bm x})^\top
\nabla u(\bm Z_r^{s,\bm x})
\right]dr\\
&+\frac12\int_s^t
\alpha(r)\,
\mathbb E\!\left[
(\sigma\sigma^\top)(\bm Z_r^{s,\bm x})
:\nabla^2u(\bm Z_r^{s,\bm x})
\right]dr.
\end{aligned}
\label{eq:v3-fixed-window-dynkin}
\end{equation}
\end{lemma}

\Beginproof{}
For $R>\|\bm x\|$, define
\[
\tau_R=\inf\{r\geq s:
\bm Z_r^{s,\bm x}\notin B(\bm 0,R)\}.
\]
On the closed ball $\overline{B}(\bm 0,R)$, the coefficients and the first
two derivatives of $u$ are bounded, since they are continuous on compact
sets. Applying It\^o's formula up to $t\wedge\tau_R$
gives
\[
\begin{aligned}
u(\bm Z_{t\wedge\tau_R}^{s,\bm x})
={}&u(\bm x)
+\int_s^{t\wedge\tau_R}
\bm b(\bm Z_r^{s,\bm x})^\top
\nabla u(\bm Z_r^{s,\bm x})\,dr\\
&+\frac12\int_s^{t\wedge\tau_R}
\alpha(r)
(\sigma\sigma^\top)(\bm Z_r^{s,\bm x})
:\nabla^2u(\bm Z_r^{s,\bm x})\,dr\\
&+\int_s^{t\wedge\tau_R}
\sqrt{\alpha(r)}\,
\nabla u(\bm Z_r^{s,\bm x})^\top
\sigma(\bm Z_r^{s,\bm x})\,d\bm B_r.
\end{aligned}
\]
On $B(\bm 0,R)$, the stochastic-integral integrand is bounded. Since
$\alpha(r)\leq\alpha_{\max}$ and $t-s<\infty$, the stopped stochastic
integral is square-integrable and has expectation zero. Taking expectations
in the stopped identity  and applying Fubini's theorem, we obtain
therefore gives
\[
\begin{aligned}
\mathbb E u(\bm Z_{t\wedge\tau_R}^{s,\bm x})
={}&u(\bm x)
+\int_s^t\mathbb E\!\left[
\mathbf 1_{\{r\leq\tau_R\}}
\bm b(\bm Z_r^{s,\bm x})^\top
\nabla u(\bm Z_r^{s,\bm x})
\right]dr\\
&+\frac12\int_s^t\alpha(r)\mathbb E\!\left[
\mathbf 1_{\{r\leq\tau_R\}}
(\sigma\sigma^\top)(\bm Z_r^{s,\bm x})
:\nabla^2u(\bm Z_r^{s,\bm x})
\right]dr.
\end{aligned}
\]
It remains to remove that the stopping term can be safely removed by letting $R\to\infty$. We use the polynomial derivative
bounds and Lemma~\ref{lem:v3-restarted-state-moments} to control the stopped
terminal term and the two time integrals. The fundamental theorem of calculus
and the gradient bound give
\begin{equation}
|u(\bm y)| \leq |u(\bm 0)| + \|\bm y\|\int_0^1 \|\nabla u(\theta\bm y)\|\,d\theta \leq C\polybrac{\bm y}^{a+1}.
\label{eq:v3-fixed-window-u-growth}
\end{equation}
Linear growth of $\bm b$ and $\sigma$ gives
\begin{equation}
|\bm b(\bm y)^\top\nabla u(\bm y)| \leq C\polybrac{\bm y}^{a+1}, \qquad |(\sigma\sigma^\top)(\bm y):\nabla^2u(\bm y)| \leq C\polybrac{\bm y}^{b+2}.
\label{eq:v3-fixed-window-dynkin-integrand-growth}
\end{equation}
These bounds control, respectively, the stopped terminal term, the
drift integrand, and the second-order integrand. Since $t-s\leq S$,
Lemma~\ref{lem:v3-restarted-state-moments}, applied at the orders $a+1$ and
$b+2$, gives
\[
\mathbb E\sup_{s\leq r\leq t}
\polybrac{\bm Z_r^{s,\bm x}}^{a+1}<\infty,
\qquad
\mathbb E\sup_{s\leq r\leq t}
\polybrac{\bm Z_r^{s,\bm x}}^{b+2}<\infty.
\]

The assumptions of Lemma~\ref{lem:v3-restarted-state-moments} ensure
nonexplosion on finite intervals, so $\tau_R\to\infty$ almost surely and
$\bm Z_{t\wedge\tau_R}^{s,\bm x}\to\bm Z_t^{s,\bm x}$ almost surely.
By \eqref{eq:v3-fixed-window-u-growth}, since $t\wedge\tau_R\in[s,t]$,
\[
|u(\bm Z_{t\wedge\tau_R}^{s,\bm x})|
\leq C\polybrac{\bm Z_{t\wedge\tau_R}^{s,\bm x}}^{a+1}
\leq C\sup_{s\leq r\leq t}
\polybrac{\bm Z_r^{s,\bm x}}^{a+1}.
\]
The final random variable is integrable by
Lemma~\ref{lem:v3-restarted-state-moments}, as stated in
\eqref{eq:v3-restarted-state-moments}. Therefore the dominated convergence
theorem gives
\[
\mathbb E u(\bm Z_{t\wedge\tau_R}^{s,\bm x})
\longrightarrow
\mathbb E u(\bm Z_t^{s,\bm x})=P_{s,t}u(\bm x).
\]

The same argument applies to the two time-integral terms. By
\eqref{eq:v3-fixed-window-dynkin-integrand-growth}, their integrands are
bounded respectively by
\[
C\polybrac{\bm Z_r^{s,\bm x}}^{a+1}
\qquad\text{and}\qquad
C\polybrac{\bm Z_r^{s,\bm x}}^{b+2}.
\]
The corresponding path suprema are integrable by
\eqref{eq:v3-restarted-state-moments}. Since
$\mathbf 1_{\{r\leq\tau_R\}}\to1$ almost surely, dominated convergence
therefore removes the stopping from both time integrals. Letting $R\to\infty$
in the stopped expectation identity gives
\eqref{eq:v3-fixed-window-dynkin}.
\Endproof

The $p$-moment consequence of
Lemma~\ref{lem:window-two-bound-taylor} already provides the
stochastic-approximation estimate needed below. It remains to record the
corresponding first-order approximation over one SDE interval.
Proposition~\ref{prop:v3-primitive-window-verification} will apply these
two estimates to the same backward function.

\begin{lemma}[One-step SDE approximation]
\label{lem:v3-sde-one-step-expansion}
Under the assumptions of Lemma~\ref{lem:v3-restarted-state-moments}, fix
$v,\kappa\geq0$ and let $u\in C^2(\mathbb R^d)$ satisfy, for some
$C_u<\infty$,
\begin{equation}
\|\nabla u(\bm y)\|
\leq C_u\polybrac{\bm y}^{v},
\qquad
\|\nabla^2u(\bm y)\|
\leq C_u\polybrac{\bm y}^{v+\kappa}.
\label{eq:v3-sde-one-step-derivative-growth}
\end{equation}
Set $q:=v+\kappa+2$. Then, uniformly in $k$ and $\bm x$,
\begin{equation}
\left|
P_{t_k,t_{k+1}}u(\bm x)-u(\bm x)
-\alpha_k\bm b(\bm x)^\top\nabla u(\bm x)
\right|
\leq C\alpha_k^2\polybrac{\bm x}^{q},
\label{eq:v3-sde-one-step-expansion}
\end{equation}
where $C$ may depend on the fixed coefficient data, $v,\kappa$, $C_u$, and
$\alpha_{\max}$, but not on $k$ or $\bm x$.
\end{lemma}
\Beginproof{}
Set $b_u(\bm y):=\bm b(\bm y)^\top\nabla u(\bm y)$. The derivative
hypotheses satisfy Lemma~\ref{lem:v3-fixed-window-dynkin} with $a=v$ and
$b=v+\kappa$. Since $\alpha_k\leq\alpha_{\max}$ and
$\alpha(t)=\alpha_k$ on $[t_k,t_{k+1})$, Dynkin's formula gives
\begin{equation}
\begin{aligned}
P_{t_k,t_{k+1}}u(\bm x)
={}&u(\bm x)
+\int_0^{\alpha_k}\mathbb E[b_u(\bm Z_{t_k+r}^{t_k,\bm x})]\,dr\\
&+\frac{\alpha_k}{2}\int_0^{\alpha_k}\mathbb E\left[
(\sigma\sigma^\top)(\bm Z_{t_k+r}^{t_k,\bm x})
:\nabla^2u(\bm Z_{t_k+r}^{t_k,\bm x})\right]dr.
\end{aligned}
\label{eq:v3-sde-one-step-dynkin}
\end{equation}
Since $\alpha_kb_u(\bm x)=\int_0^{\alpha_k}b_u(\bm x)\,dr$,
subtracting $u(\bm x)+\alpha_kb_u(\bm x)$ gives the exact identity
\begin{equation}
\begin{aligned}
&P_{t_k,t_{k+1}}u(\bm x)-u(\bm x)-\alpha_kb_u(\bm x)\\
&\quad=\int_0^{\alpha_k}\mathbb E\left[
b_u(\bm Z_{t_k+r}^{t_k,\bm x})-b_u(\bm x)\right]dr\\
&\qquad+\frac{\alpha_k}{2}\int_0^{\alpha_k}\mathbb E\left[
(\sigma\sigma^\top)(\bm Z_{t_k+r}^{t_k,\bm x})
:\nabla^2u(\bm Z_{t_k+r}^{t_k,\bm x})\right]dr.
\end{aligned}
\label{eq:v3-sde-one-step-error-identity}
\end{equation}
We bound the two terms on the right-hand side in turn. To control the first,
we use the displacement of the restarted diffusion. The coefficient
assumptions give
\begin{equation}
\|\bm b(\bm y)\|+\|\sigma(\bm y)\|_F\leq C\polybrac{\bm y},
\qquad
\|D\bm b(\bm y)\|\leq C.
\label{eq:v3-one-step-coefficient-bounds}
\end{equation}
All moment estimates below use
\eqref{eq:v3-restarted-state-moments} with $S=\alpha_{\max}$.
For $0\leq r\leq\alpha_k$, \eqref{eq:v3-associated-sde} gives
\[
\bm Z_{t_k+r}^{t_k,\bm x}-\bm x
=\int_0^r\bm b(\bm Z_{t_k+s}^{t_k,\bm x})\,ds
+\sqrt{\alpha_k}\int_0^r
\sigma(\bm Z_{t_k+s}^{t_k,\bm x})\,d\bm B_{t_k+s}.
\]
The Cauchy--Schwarz inequality, It\^o isometry, linear growth of the
coefficients, and the restarted-SDE moment bound
\eqref{eq:v3-restarted-state-moments} yield
\begin{equation}
\begin{aligned}
\mathbb E\|\bm Z_{t_k+r}^{t_k,\bm x}-\bm x\|^2
&\leq
2\mathbb E\left\|\int_0^r\bm b(\bm Z_{t_k+s}^{t_k,\bm x})\,ds\right\|^2
+2\alpha_k\mathbb E\left\|\int_0^r\sigma(\bm Z_{t_k+s}^{t_k,\bm x})\,d\bm B_{t_k+s}\right\|^2\\
&\leq
2r\int_0^r\mathbb E\|\bm b(\bm Z_{t_k+s}^{t_k,\bm x})\|^2\,ds
+2\alpha_k\int_0^r
\mathbb E\|\sigma(\bm Z_{t_k+s}^{t_k,\bm x})\|_F^2\,ds\\
&\leq Cr\int_0^r\mathbb E\polybrac{\bm Z_{t_k+s}^{t_k,\bm x}}^2\,ds
+C\alpha_k\int_0^r\mathbb E\polybrac{\bm Z_{t_k+s}^{t_k,\bm x}}^2\,ds\\
&\leq C(r^2+\alpha_k r)\polybrac{\bm x}^{2}.
\end{aligned}
\label{eq:v3-primitive-one-step-displacement}
\end{equation}
The product rule gives
\[
Db_u(\bm y)
=D\bm b(\bm y)^\top\nabla u(\bm y)
+\nabla^2u(\bm y)\bm b(\bm y).
\]
Using \eqref{eq:v3-sde-one-step-derivative-growth} and
\eqref{eq:v3-one-step-coefficient-bounds},
\begin{equation}
\begin{aligned}
\|Db_u(\bm y)\|
&\leq \|D\bm b(\bm y)\|\,\|\nabla u(\bm y)\|
+\|\nabla^2u(\bm y)\|\,\|\bm b(\bm y)\|\\
&\leq C\polybrac{\bm y}^{v}
+C\polybrac{\bm y}^{v+\kappa+1}\\
&\leq C\polybrac{\bm y}^{v+\kappa+1},
\end{aligned}
\label{eq:v3-sde-drift-derivative-bound}
\end{equation}
where the last inequality uses $\kappa\geq0$ and
$\polybrac{\bm y}\geq1$. The fundamental theorem of calculus along the
segment from $\bm x$ to $\bm y$ gives
\[
b_u(\bm y)-b_u(\bm x)
=\int_0^1Db_u\!\left(\bm x+\theta(\bm y-\bm x)\right)^\top
(\bm y-\bm x)\,d\theta.
\]
Together with the polynomial translation inequality
\eqref{eq:polybrac-translation}, this implies
\[
|b_u(\bm y)-b_u(\bm x)|
\leq C\|\bm y-\bm x\|
\left\{\polybrac{\bm x}^{v+\kappa+1}
+\polybrac{\bm y}^{v+\kappa+1}\right\}.
\]
Set $\bm y=\bm Z_{t_k+r}^{t_k,\bm x}$. Cauchy--Schwarz gives
\begin{equation}
\begin{aligned}
&\mathbb E|b_u(\bm Z_{t_k+r}^{t_k,\bm x})-b_u(\bm x)|\\
&\quad\leq C\left(\mathbb E
\|\bm Z_{t_k+r}^{t_k,\bm x}-\bm x\|^2\right)^{1/2}\\
&\qquad\quad\times\left\{\polybrac{\bm x}^{v+\kappa+1}
+\left(\mathbb E\polybrac{\bm Z_{t_k+r}^{t_k,\bm x}}^{2(v+\kappa+1)}\right)^{1/2}\right\}.
\end{aligned}
\end{equation}
By \eqref{eq:v3-primitive-one-step-displacement}, the first factor is at
most $C(r+\sqrt{\alpha_k r})\polybrac{\bm x}$, using
$\sqrt{r^2+\alpha_k r}\leq r+\sqrt{\alpha_k r}$. By
\eqref{eq:v3-restarted-state-moments} with moment order
$2(v+\kappa+1)$, the second factor is at most
$C\polybrac{\bm x}^{v+\kappa+1}$. Thus
\begin{equation}
\mathbb E|b_u(\bm Z_{t_k+r}^{t_k,\bm x})-b_u(\bm x)|
\leq C(r+\sqrt{\alpha_k r})\polybrac{\bm x}^{v+\kappa+2}.
\label{eq:v3-sde-drift-variation-bound}
\end{equation}
For the second term in \eqref{eq:v3-sde-one-step-error-identity}, linear
growth of $\sigma$ and the Hessian bound give
\[
|(\sigma\sigma^\top)(\bm y):\nabla^2u(\bm y)|
\leq C\polybrac{\bm y}^{v+\kappa+2}.
\]
Using \eqref{eq:v3-restarted-state-moments} with moment order $q$ yields
\begin{equation}
\mathbb E\left|
(\sigma\sigma^\top)(\bm Z_{t_k+r}^{t_k,\bm x})
:\nabla^2u(\bm Z_{t_k+r}^{t_k,\bm x})\right|
\leq C\polybrac{\bm x}^{q}.
\label{eq:v3-sde-diffusion-term-bound}
\end{equation}
Substituting \eqref{eq:v3-sde-drift-variation-bound} and
\eqref{eq:v3-sde-diffusion-term-bound} into
\eqref{eq:v3-sde-one-step-error-identity} gives
\[
\begin{aligned}
&|P_{t_k,t_{k+1}}u(\bm x)-u(\bm x)-\alpha_kb_u(\bm x)|\\
&\quad\leq C\polybrac{\bm x}^{q}
\left\{\int_0^{\alpha_k}(r+\sqrt{\alpha_k r})\,dr
+\alpha_k\int_0^{\alpha_k}dr\right\}.
\end{aligned}
\]
Here the factor $1/2$ has been absorbed into $C$, and
\[
\int_0^{\alpha_k}(r+\sqrt{\alpha_k r})\,dr=\frac76\alpha_k^2,
\qquad
\alpha_k\int_0^{\alpha_k}dr=\alpha_k^2.
\]
Since $b_u(\bm x)=\bm b(\bm x)^\top\nabla u(\bm x)$, absorbing the
numerical constants into $C$ proves \eqref{eq:v3-sde-one-step-expansion}.
\Endproof


\begin{proposition}[Primitive finite window verification]
\label{prop:v3-primitive-window-verification}
Under Assumption~\ref{assum:v3-primitive-finite-window}, for
$m\leq n\leq N_T(m)$ satisfying
$\mathbb{E}\polybrac{\bm X_m}^{p}<\infty$,
\begin{equation}
\left|\mathbb{E}V(\bm X_n)
-\mathbb{E}\!\left[P_{t_m,t_n}V(\bm X_m)\right]\right|
\leq C_T\mathbb{E}\polybrac{\bm X_m}^{p}
\int_{t_m}^{t_n}\alpha(s)\,ds.
\label{eq:v3-primitive-window-comparison}
\end{equation}
If, in addition, the stepsizes are nonincreasing, then
\[
\left|\mathbb{E}V(\bm X_n)
-\mathbb{E}\!\left[P_{t_m,t_n}V(\bm X_m)\right]\right|
\leq C_T\alpha_m\mathbb{E}\polybrac{\bm X_m}^{p}.
\]
If, in addition, $M_p:=\sup_{k\geq0}\mathbb{E}\polybrac{\bm X_k}^{p}<\infty$,
then

\begin{equation}
\mathbb{E}V(\bm X_n)
\leq \mathbb{E}\!\left[P_{t_m,t_n}V(\bm X_m)\right]
+C_TM_p\int_{t_m}^{t_n}\alpha(s)\,ds.
\label{eq:v3-primitive-window-endpoint-bound}
\end{equation}
Therefore the finite window approximation
\eqref{eq:finite-window-sde-approximation} holds with
$\epsilon_m=0$ and $\rho(s)=C_TM_p\alpha(s)$.
\end{proposition}
\Beginproof{ of Proposition~\ref{prop:v3-primitive-window-verification}}
We first establish the discrete fixed-window $p$-moment bound.
We then apply the one-step SDE approximation and the $p$-moment
specialization of Lemma~\ref{lem:window-two-bound-taylor} to the same
backward function. Both approximate the same first-order expression.
Finally, we telescope the resulting one-step defect over the finite window
to obtain \eqref{eq:v3-primitive-window-comparison}.

For the remainder of the proof set
\(
S_T:=T+\alpha_{\max}.
\)
If $m\leq k<n\leq N_T(m)$, all restarted-SDE and backward-semigroup
intervals used below have length at most $S_T$. Hence
Lemmas~\ref{lem:v3-restarted-state-moments} and
\ref{lem:v3-polynomial-window-control} apply uniformly with $S=S_T$.

\emph{Discrete window stability.}
From $\bm X_{k+1}=\bm X_k+\alpha_k\bm H_{k+1}$ and weighted convexity,
\[
\begin{aligned}
\mathbb E[\polybrac{\bm X_{k+1}}^p\mid\mathcal F_k]
&\leq
(1+\alpha_k)^{p-1}
\left\{
\polybrac{\bm X_k}^p
+\alpha_k\mathbb E[\|\bm H_{k+1}\|^p\mid\mathcal F_k]
\right\}\\
&\leq
(1+\alpha_k)^{p-1}(1+M\alpha_k)
\polybrac{\bm X_k}^p\\
&\leq
\exp\{(p-1+M)\alpha_k\}\polybrac{\bm X_k}^p.
\end{aligned}
\]
Iterating and using $t_k-t_m\leq T+\alpha_{\max}$ gives

\begin{equation}
\mathbb E[\polybrac{\bm X_k}^p\mid\mathcal F_m]
\leq C_T\polybrac{\bm X_m}^p,
\qquad m\leq k\leq N_T(m),
\label{eq:v3-discrete-window-p-moment}
\end{equation}
where one may take
$C_T=\exp\{(p-1+M)(T+\alpha_{\max})\}$.

\emph{One-step SDE--stochastic-approximation comparison.}
Fix $m\leq k<n\leq N_T(m)$ and write
$u=u_{k+1,n}=P_{t_{k+1},t_n}V$, with
$q:=v+\kappa+2\leq p$. By
\eqref{eq:v3-polynomial-window-derivative-bounds}, the hypotheses of both
the SDE approximation and the conditional Taylor bound hold uniformly on
the window, with derivative constant $C_T$.

Applying Lemma~\ref{lem:v3-sde-one-step-expansion} at
$\bm x=\bm X_k$ gives
\[
\left|
P_{t_k,t_{k+1}}u(\bm X_k)-u(\bm X_k)
-\alpha_k\bm b(\bm X_k)^\top\nabla u(\bm X_k)
\right|
\leq C_T\alpha_k^2\polybrac{\bm X_k}^{q}.
\]
For the stochastic-approximation estimate, apply
\eqref{eq:window-taylor-conditional-p-moment-bound} with
$\gamma_1=v$, $\gamma_2=v+\kappa$, $\bm x=\bm X_k$,
$h=\alpha_k$, $\bm H=\bm H_{k+1}$, and $\mathcal F=\mathcal F_k$.
The oracle bound \eqref{eq:v3-primitive-oracle-p-moment} supplies the
required $p$-moment, since $p\geq q=\gamma_2+2$. Using
$\bm X_{k+1}=\bm X_k+\alpha_k\bm H_{k+1}$ and
$\mathbb E[\bm H_{k+1}\mid\mathcal F_k]=\bm b(\bm X_k)$ gives
\[
\left|
\mathbb E[u(\bm X_{k+1})\mid\mathcal F_k]-u(\bm X_k)
-\alpha_k\bm b(\bm X_k)^\top\nabla u(\bm X_k)
\right|
\leq C_T\alpha_k^2\polybrac{\bm X_k}^{q}.
\]
Both quantities are therefore within
$C_T\alpha_k^2\polybrac{\bm X_k}^{p}$ of the same first-order expression
\[
u(\bm X_k)+\alpha_k\bm b(\bm X_k)^\top\nabla u(\bm X_k),
\]
where we used $q\leq p$ and enlarged $C_T$ if needed. The triangle
inequality gives, after enlarging $C_T$,
\[
\left|
\mathbb E[u(\bm X_{k+1})\mid\mathcal F_k]
-P_{t_k,t_{k+1}}u(\bm X_k)
\right|
\leq C_T\alpha_k^2\polybrac{\bm X_k}^{p}.
\]
Using $u=u_{k+1,n}$ and
$P_{t_k,t_{k+1}}u_{k+1,n}=u_{k,n}$ gives the one-step defect:

\begin{equation}
\left|
\mathbb E[u_{k+1,n}(\bm X_{k+1})\mid\mathcal F_k]
-u_{k,n}(\bm X_k)
\right|
\leq C_T\alpha_k^2\polybrac{\bm X_k}^{p}.
\label{eq:v3-primitive-one-step-defect}
\end{equation}

\emph{Passage to the finite window.}
Since $u_{n,n}=V$ and
$u_{k,n}=P_{t_k,t_{k+1}}u_{k+1,n}$, telescoping and the tower property give

\[
\begin{aligned}
&\mathbb E[V(\bm X_n)\mid\mathcal F_m]
-P_{t_m,t_n}V(\bm X_m)\\
&\qquad=
\sum_{k=m}^{n-1}
\mathbb E\!\left[
\mathbb E[u_{k+1,n}(\bm X_{k+1})\mid\mathcal F_k]
-u_{k,n}(\bm X_k)
\,\middle|\,\mathcal F_m
\right].
\end{aligned}
\]
Using \eqref{eq:v3-primitive-one-step-defect} and
\eqref{eq:v3-discrete-window-p-moment} therefore gives
\begin{equation}
\begin{aligned}
\left|
\mathbb E[V(\bm X_n)\mid\mathcal F_m]
-P_{t_m,t_n}V(\bm X_m)
\right|
&\leq C_T\polybrac{\bm X_m}^{p}
\sum_{k=m}^{n-1}\alpha_k^2\\
&=C_T\polybrac{\bm X_m}^{p}
\int_{t_m}^{t_n}\alpha(s)\,ds.
\end{aligned}
\label{eq:v3-primitive-window-conditional}
\end{equation}
where the equality uses $t_{k+1}-t_k=\alpha_k$ and the piecewise-constant
definition of $\alpha(s)$. Taking expectations and using
$|\mathbb EY|\leq\mathbb E|Y|$ gives
\eqref{eq:v3-primitive-window-comparison}. If the stepsizes are
nonincreasing, then
\[
\int_{t_m}^{t_n}\alpha(s)\,ds
\leq\alpha_m(t_n-t_m)
\leq(T+\alpha_{\max})\alpha_m,
\]
so the preceding estimate gives the stated
$C_T\alpha_m\polybrac{\bm X_m}^{p}$ bound after enlarging $C_T$. Taking
expectations gives the endpoint assertion of the proposition.

If $M_p=\sup_k\mathbb E\polybrac{\bm X_k}^{p}<\infty$, the comparison also
gives

\[
\mathbb EV(\bm X_n)
\leq\mathbb E[P_{t_m,t_n}V(\bm X_m)]
+C_TM_p\int_{t_m}^{t_n}\alpha(s)\,ds.
\]
Thus \eqref{eq:finite-window-sde-approximation} holds with
$\epsilon_m=0$ and $\rho(s)=C_TM_p\alpha(s)$, as required. This proves all
claims of the proposition.
\Endproof

\section{Proof of the Rate Transfer Theorem}\label{sec:proof-main-rate-transfer}
This appendix combines the finite-window comparison with the SDE Lyapunov envelope and restart argument to prove the Rate Transfer Theorem.

The proof preserves the modular development of the transfer argument. The
technical lemmas establish the primitive finite-window comparison; that
comparison supplies the input to the abstract restart theorem; the final proof
below applies the abstract theorem under the Lyapunov hypothesis in
Theorem~\ref{thm:main-rate-transfer}. We give the abstract transfer result,
the primitive verification, and their proofs in separate modules.

The main result proved in this appendix section is
Theorem~\ref{thm:main-rate-transfer}. Its proof combines the primitive
finite-window verification with the abstract SDE transfer result.

The SDE Transfer Theorem has two main inputs: a Lyapunov bound for the
associated SDE and a finite window approximation comparing the stochastic
approximation with that SDE. In applications, these are the two main
conditions that must be verified.

We now state one endpoint transfer theorem for the stochastic approximation
over fixed diffusion time windows. Its two inputs are a Lyapunov generator
inequality for the restarted SDE and a mean finite window approximation for
the stochastic approximation. For the endpoint transfer machinery assume
\[
\alpha_n>0,\qquad \alpha_{\max}:=\sup_n\alpha_n<\infty,\qquad
\sum_{n=0}^{\infty}\alpha_n=\infty.
\]
Recall the restarted process $\bm Z^{s,\bm x}$ and its evolution family
$P_{s,t}$ defined in Section~\ref{sec:v3-associated-sde}. Its composition
property \eqref{eq:P-composition} will be used repeatedly. The scalar Lyapunov
factor is

\[
R(s,t)=\exp\left\{-\int_s^t\lambda(r)\,dr\right\}.
\]
Its composition property is given by \eqref{eq:R-composition}.

The endpoint transfer route combines a continuous estimate with a finite
window comparison:
\[
\text{continuous estimate}
\;+
\text{finite window comparison}
\;\Longrightarrow\;
\text{stochastic approximation estimate}.
\]
We first give the abstract restart theorem, then verify its comparison
hypothesis from primitive model conditions.

\subsection{Lyapunov-envelope and endpoint transfer}

The main result of this subsection is the SDE transfer theorem given
below. It
combines a Lyapunov bound for the associated SDE with a finite window
approximation between the stochastic approximation and that SDE. We state
the result first and then develop the two ingredients used in its proof.

\begin{theorem}[SDE Transfer Theorem]\label{thm:finite-window-sde-rate-transfer}
Assume $V\in C^2(\mathbb R^d)$, $V\geq0$, and
$\mathbb{E}V(\bm X_0)<\infty$. Assume positive stepsizes with
$\alpha_{\max}:=\sup_n\alpha_n<\infty$ and $\sum_n\alpha_n=\infty$, and suppose that the restarted SDE
is nonexplosive on finite intervals and that It\^o's formula applies locally.

Suppose, for deterministic nonnegative $\lambda$, $g$, and $\rho$ that are
locally integrable,
\begin{equation}\label{eq:finite-window-sde-generator}
\mathcal L_tV(\bm x)\leq-\lambda(t)V(\bm x)+g(t).
\end{equation}
Assume that, at the required restart starts, the finite window
approximation
\begin{equation}\label{eq:finite-window-sde-approximation}
\begin{aligned}
\mathbb E V(\bm X_n)
\leq
\mathbb E\!\left[P_{t_m,t_n}V(\bm X_m)\right]
+\epsilon_m\mathbb E V(\bm X_m)
+\int_{t_m}^{t_n}\rho(s)\,ds,
\qquad m\leq n\leq N_T(m),
\end{aligned}
\end{equation}
holds, where $\epsilon_m\geq0$ is deterministic.

Set $m_0=0$ and $m_{j+1}=N_T(m_j)$, and put
$R_j=R(t_{m_j},t_{m_{j+1}})$. Suppose that
$\sum_{j\geq0}\epsilon_{m_j}/R_j<\infty$. With
\[
C_\epsilon=\exp\left\{\sum_{j\geq0}\frac{\epsilon_{m_j}}{R_j}\right\},
\]
assume in addition that $0\leq\lambda(t)\leq\bar\lambda<\infty$ and set
\(
K_T=\exp\{\bar\lambda(T+\alpha_{\max})\}.
\)
Then, for every $n\geq0$,

\begin{equation}
\mathbb E V(\bm X_n)\leq C_\epsilon\Bigg[R(0,t_n)\mathbb E V(\bm X_0)+\int_0^{t_n}R(s,t_n)\{g(s)+K_T\rho(s)\}\,ds\Bigg].
\label{eq:finite-window-sde-rate-transfer-bound}
\end{equation}
\end{theorem}

The proof uses two results. First,
Lemma~\ref{lem:sde-convergence-envelope} converts the generator Lyapunov
bound into a finite time estimate for the restarted SDE. This lemma formalizes
the argument that we outlined in Section~2.4. Next,
Proposition~\ref{prop:finite-window-block-residual-transfer} shows how this
continuous estimate and a general finite window residual can be composed over
successive restart windows. This proposition is slightly more general than Theorem~\ref{thm:finite-window-sde-rate-transfer}, but it does not make explicit the underlying connection with stochastic calculus. After establishing these two results, we return to
prove the theorem by specializing the residual to the finite window
approximation above.

The following lemma isolates the main Lyapunov bound supplied by the
stochastic differential equation. It converts the generator inequality into
the finite time evolution-operator estimate used in the transfer argument.

\begin{lemma}[Finite-time SDE Lyapunov envelope]\label{lem:sde-convergence-envelope}
Let $V\in C^2(\mathbb R^d)$ with $V\geq0$. Suppose that the restarted SDE is nonexplosive on
finite intervals and that It\^o's formula applies locally. Suppose that
$\lambda$ and $g$ are deterministic, nonnegative, locally integrable
functions satisfying

\[
\mathcal L_tV(\bm x)\leq-\lambda(t)V(\bm x)+g(t).
\]
Then, for $s\leq t$,
\begin{equation}
P_{s,t}V(\bm x)
\leq
R(s,t)V(\bm x)+\int_s^tR(r,t)g(r)\,dr.
\end{equation}
\end{lemma}
\Beginproof{}
Write

\[
A_r=\int_s^r\lambda(q)\,dq.
\]
Since $d e^{A_r}=e^{A_r}\lambda(r)\,dr$, It\^o's formula gives
\[
 dV(\bm Z_r^{s,\bm x})
=\mathcal L_rV(\bm Z_r^{s,\bm x})\,dr+dM_r,
\]
where $M$ is a local martingale. Since $e^{A_r}$ is deterministic and locally
bounded on every finite interval, define
\[
d\widetilde M_r:=e^{A_r}\,dM_r.
\]
Then $\widetilde M$ is again a local martingale.

Since the integrating factor has finite variation, its quadratic covariation
with $V(\bm Z_r^{s,\bm x})$ is zero. The product rule and the Lyapunov
inequality therefore give
\[
\begin{aligned}
d\left\{e^{A_r}V(\bm Z_r^{s,\bm x})\right\}
&=e^{A_r}\,dV(\bm Z_r^{s,\bm x})
 +V(\bm Z_r^{s,\bm x})\,d e^{A_r}\\
&=e^{A_r}\left\{\mathcal L_rV(\bm Z_r^{s,\bm x})
 +\lambda(r)V(\bm Z_r^{s,\bm x})\right\}\,dr
 +d\widetilde M_r\\
&\leq e^{A_r}g(r)\,dr+d\widetilde M_r.
\end{aligned}
\]
Indeed, the last inequality is precisely
$\mathcal L_rV\leq-\lambda(r)V+g(r)$. Let $\tau_R$ be the exit time of
$\bm Z^{s,\bm x}$ from the ball of radius $R$, and let $(\sigma_N)$ be a
localizing sequence such that $\widetilde M^{\sigma_N}$ is a martingale. Set
$\eta_{R,N}=\tau_R\wedge\sigma_N$.

Stopping, taking expectations, and using $g\geq0$ gives
\[
\mathbb{E}\!\left[
e^{A_{t\wedge\eta_{R,N}}}V(\bm Z_{t\wedge\eta_{R,N}}^{s,\bm x})
\right]
\leq V(\bm x)+\int_s^t e^{A_r}g(r)\,dr.
\]
Since $V\geq0$, Fatou's lemma as $N\to\infty$ gives
\[
\mathbb{E}\!\left[
e^{A_{t\wedge\tau_R}}V(\bm Z_{t\wedge\tau_R}^{s,\bm x})
\right]
\leq V(\bm x)+\int_s^t e^{A_r}g(r)\,dr.
\]
By finite time nonexplosion, $\tau_R\to\infty$ almost surely as
$R\to\infty$. A second application of Fatou therefore gives

\[\mathbb{E}\!\left[e^{A_t}V(\bm Z_t^{s,\bm x})\right] = e^{A_t}\mathbb{E}\!\left[V(\bm Z_t^{s,\bm x})\right] = e^{A_t}P_{s,t}V(\bm x) \leq V(\bm x)+\int_s^t e^{A_r}g(r)\,dr.\]
Thus

\begin{equation}\label{eq:sde-convergence-envelope}
P_{s,t}V(\bm x)
\leq R(s,t)V(\bm x)+\int_s^tR(r,t)g(r)\,dr.
\end{equation}
\Endproof

Lemma~\ref{lem:sde-convergence-envelope} supplies the continuous time
estimate used in the blockwise transfer argument.

The next proposition gives a general blockwise transfer result. Given a Lyapunov bound for the restarted SDE, it shows how a one-sided finite window comparison over successive restart windows gives a global bound for the stochastic approximation.

\begin{proposition}[General block residual finite window transfer]\label{prop:finite-window-block-residual-transfer}
Assume $V\in C^2(\mathbb R^d)$, $V\geq0$, and
$\mathbb{E}V(\bm X_{m_0})<\infty$. Assume positive stepsizes with
$\alpha_{\max}:=\sup_n\alpha_n<\infty$ and $\sum_n\alpha_n=\infty$, and suppose that the restarted SDE
is nonexplosive on finite intervals and that It\^o's formula applies locally.

Suppose, for deterministic nonnegative $\lambda$ and $g$ that are locally
integrable,
\[
\mathcal L_tV(\bm x)\leq-\lambda(t)V(\bm x)+g(t).
\]
Starting from $m_0$, set $m_{j+1}=N_T(m_j)$. Assume that, at each
required restart start $m$, for every $m\leq n\leq N_T(m)$,
\begin{equation}\label{eq:finite-window-V-approximation}
\mathbb{E}V(\bm X_n)
\leq
\mathbb{E}\!\left[P_{t_m,t_n}V(\bm X_m)\right]
+\epsilon_m\mathbb{E}V(\bm X_m)+r_{m,n},
\qquad m\leq n\leq N_T(m),
\end{equation}
where $\epsilon_m,r_{m,n}\geq0$ are deterministic. Put
$R_j=R(t_{m_j},t_{m_{j+1}})$ and suppose
$\sum_{j\geq0}\epsilon_{m_j}/R_j<\infty$. With

\[
C_\epsilon=\exp\left\{\sum_{j\geq0}\frac{\epsilon_{m_j}}{R_j}\right\},
\]
for $n>m_0$, choose $J$ so that $m_J<n\leq m_{J+1}$. Then
\[
\mathbb{E}V(\bm X_n)\leq C_\epsilon\Bigg[R(t_{m_0},t_n)\mathbb{E}V(\bm X_{m_0})+\int_{t_{m_0}}^{t_n}R(s,t_n)g(s)\,ds+\sum_{j=0}^{J-1}R(t_{m_{j+1}},t_n)r_{m_j,m_{j+1}}+r_{m_J,n}\Bigg].
\]

\end{proposition}
\Beginproof{}
By Lemma~\ref{lem:sde-convergence-envelope}, we have the SDE bound
\eqref{eq:sde-convergence-envelope}.
Writing out \eqref{eq:sde-convergence-envelope} over the block from
$m_j$ to $m_{j+1}$ and taking expectations gives
\[
\mathbb E\!\left[
P_{t_{m_j},t_{m_{j+1}}}V(\bm X_{m_j})
\right]
\leq
R(t_{m_j},t_{m_{j+1}})
\mathbb E V(\bm X_{m_j})
+
\int_{t_{m_j}}^{t_{m_{j+1}}}
R(s,t_{m_{j+1}})g(s)\,ds.
\]
Applying this to the finite window approximation
\eqref{eq:finite-window-V-approximation} with $m=m_j$ and
$n=m_{j+1}$ gives
\[
\mathbb E V(\bm X_{m_{j+1}})
\leq
\left\{
R(t_{m_j},t_{m_{j+1}})
+\epsilon_{m_j}
\right\}
\mathbb E V(\bm X_{m_j})
+
\int_{t_{m_j}}^{t_{m_{j+1}}}
R(s,t_{m_{j+1}})g(s)\,ds
+
r_{m_j,m_{j+1}}.
\]

To compress the notation a bit,
for $m_J<n\leq m_{J+1}$, we define $u_j=\mathbb{E}V(\bm X_{m_j})$,
$0\leq j\leq J$, and $u_{J+1}=\mathbb{E}V(\bm X_n)$. For this fixed
terminal index, set $R'_j=R_j=R(t_{m_j},t_{m_{j+1}})$, $0\leq j<J$,
and $R'_J=R(t_{m_J},t_n)$. Define the residual used in the recursion by
$r'_j=r_{m_j,m_{j+1}}$, $0\leq j<J$, and $r'_J=r_{m_J,n}$. Finally,
\[
G_j=
\int_{t_{m_j}}^{t_{m_{j+1}}}
R(s,t_{m_{j+1}})g(s)\,ds,\quad 0\leq j<J,
\qquad\text{and}\qquad
G_J=
\int_{t_{m_J}}^{t_n}R(s,t_n)g(s)\,ds.
\]
Applying the finite window approximation
\eqref{eq:finite-window-V-approximation} and the SDE envelope bound above
gives, for every $0\leq j\leq J$,

\[
u_{j+1}
\leq
(R'_j+\epsilon_{m_j})u_j+G_j+r'_j.
\]
Iterating this recursion gives the standard discrete Grönwall expansion
\[
u_{J+1}\leq{}\left\{\prod_{j=0}^{J}(R'_j+\epsilon_{m_j})\right\}u_0+\sum_{\ell=0}^{J}\left\{\prod_{j=\ell+1}^{J}(R'_j+\epsilon_{m_j})\right\}\{G_\ell+r'_\ell\}.
\]

The inequality
$R'_J=R(t_{m_J},t_n)\geq R(t_{m_J},t_{m_{J+1}})=R_J$ implies that
$\epsilon_{m_J}/R'_J\leq\epsilon_{m_J}/R_J$ and also we know
$\prod_{j=\ell}^{J}R'_j=R(t_{m_\ell},t_n)$. Thus,
\begin{align*}
\prod_{j=\ell}^{J}(R'_j+\epsilon_{m_j})
=\left(\prod_{j=\ell}^{J}R'_j\right)
\prod_{j=\ell}^{J}\left(1+\frac{\epsilon_{m_j}}{R'_j}\right)
&\leq R(t_{m_\ell},t_n)
\exp\left\{\sum_{j=\ell}^{J}\frac{\epsilon_{m_j}}{R_j}\right\}\\
&\leq C_\epsilon R(t_{m_\ell},t_n).
\end{align*}

Applying these product bounds to the Grönwall expansion gives
\[
u_{J+1}\leq C_\epsilon\Bigg[R(t_{m_0},t_n)u_0+\sum_{\ell=0}^{J-1}R(t_{m_{\ell+1}},t_n)\{G_\ell+r_{m_\ell,m_{\ell+1}}\}+G_J+r_{m_J,n}\Bigg].
\]
By the definition of $G_\ell$ and the composition property,
\[
\sum_{\ell=0}^{J-1}
R(t_{m_{\ell+1}},t_n)G_\ell+G_J
=\int_{t_{m_0}}^{t_n}R(s,t_n)g(s)\,ds.
\]
Since $u_0=\mathbb{E}V(\bm X_{m_0})$ and
$u_{J+1}=\mathbb{E}V(\bm X_n)$, this is the displayed bound.
\Endproof

The preceding proposition is the general block residual statement. We now
apply it with the residual induced by the finite window approximation in the
theorem above.
\Beginproof{ of Theorem~\ref{thm:finite-window-sde-rate-transfer}}
For $n=0$, the result is immediate. Fix $n>0$ and choose $J$ so that
$m_J<n\leq m_{J+1}$. Applying
Proposition~\ref{prop:finite-window-block-residual-transfer} with residual
$r_{m,n}=\int_{t_m}^{t_n}\rho(s)\,ds$ gives
\begin{align}
\mathbb E V(\bm X_n)
\leq C_\epsilon\Bigg[&
R(0,t_n)\mathbb E V(\bm X_0)+\int_0^{t_n}R(s,t_n)g(s)\,ds\notag\\
&+\sum_{j=0}^{J-1}
R(t_{m_{j+1}},t_n)
\int_{t_{m_j}}^{t_{m_{j+1}}}\rho(s)\,ds+\int_{t_{m_J}}^{t_n}\rho(s)\,ds
\Bigg].
\label{eq:theorem-first-bound}
\end{align}
The fixed window construction gives
\(
t_{m_{j+1}}-t_{m_j}<T+\alpha_{\max}.
\)
For each $0\leq j<J$ and $s\in[t_{m_j},t_{m_{j+1}}]$,

\[
R(s,t_{m_{j+1}})
\geq e^{-\bar\lambda(T+\alpha_{\max})}
=K_T^{-1}.
\]
For the final partial interval $s\in[t_{m_J},t_n]$, similarly
\(R(s,t_n)\geq R(s,t_{m_{J+1}})\geq K_T^{-1}\).

Consequently,
\begin{align}
&\sum_{j=0}^{J-1}
R(t_{m_{j+1}},t_n)
\int_{t_{m_j}}^{t_{m_{j+1}}}\rho(s)\,ds
+\int_{t_{m_J}}^{t_n}\rho(s)\,ds\notag\\
&\quad=
\sum_{j=0}^{J-1}
\int_{t_{m_j}}^{t_{m_{j+1}}}
\frac{R(s,t_n)}{R(s,t_{m_{j+1}})}
\rho(s)\,ds
+
\int_{t_{m_J}}^{t_n}
\frac{R(s,t_n)}{R(s,t_n)}
\rho(s)\,ds\notag\\
&\quad\leq
K_T
\sum_{j=0}^{J-1}
\int_{t_{m_j}}^{t_{m_{j+1}}}
R(s,t_n)\rho(s)\,ds
+
K_T
\int_{t_{m_J}}^{t_n}
R(s,t_n)\rho(s)\,ds\notag\\
&\quad=
K_T\int_0^{t_n}R(s,t_n)\rho(s)\,ds.
\label{eq:theorem-residual-bound}
\end{align}
Substitution of \eqref{eq:theorem-residual-bound} into
\eqref{eq:theorem-first-bound} proves the theorem.
\Endproof

The general bound
\eqref{eq:finite-window-sde-rate-transfer-bound} transfers the SDE Lyapunov
estimate with a multiplicative factor $C_\epsilon$ and an additional forcing
term $K_T\rho(s)$. In the primitive comparison of
Proposition~\ref{prop:v3-primitive-window-verification}, $\epsilon_m=0$ and
$\rho(s)=C_T M_p\alpha(s)$. Thus $C_\epsilon=1$, and the additional forcing
takes the form $C\alpha(s)$ in the specialised bound
\eqref{eq:v3-first-order-transfer-bound}.


\begin{remark}[Direct SDE envelopes]
The standard applications verify the generator drift condition, which is why
it is the continuous input to Theorem~\ref{thm:finite-window-sde-rate-transfer}.
If an application already has the integrated estimate
\[
P_{s,t}V\leq R(s,t)V+\int_s^tR(u,t)g(u)\,du,
\]
the generator hypothesis in Theorem~\ref{thm:finite-window-sde-rate-transfer}
may be replaced by this integrated SDE envelope.
\end{remark}

\Beginproof{ of Theorem~\ref{thm:main-rate-transfer}}
By \eqref{eq:v3-primitive-lyapunov-growth} and the fundamental theorem
of calculus along the segment from $\bm 0$ to $\bm x$,
$|V(\bm x)|\leq C\polybrac{\bm x}^{v+1}$. Since
$p\geq v+\kappa+2>v+1$ and $M_p<\infty$, we have
$\mathbb EV(\bm X_0)<\infty$.
Proposition~\ref{prop:v3-primitive-window-verification} and the global
$p$-th polynomial moment closure give the finite window approximation required by

Theorem~\ref{thm:finite-window-sde-rate-transfer}, with
$\epsilon_m=0$ and $\rho(s)=C_TM_p\alpha(s)$. Hence
$C_\epsilon=1$. Applying the abstract transfer theorem using the explicitly

stated Lyapunov bound, and absorbing $K_T$, $M_p$, and the fixed window
constant into $C_T$, gives the displayed bound.

This completes the proof.
\Endproof

\begin{corollary}[Rate inheritance for common stepsize schedules]
\label{cor:stepsize-rate-inheritance}
Suppose that, at the diffusion times $t_n$, a nonnegative sequence satisfies
\begin{equation}\label{eq:stepsize-rate-inheritance-envelope}
E_n\leq e^{-\lambda t_n}E_0+C\int_0^{t_n}
e^{-\lambda(t_n-s)}\alpha(s)^\delta\,ds,
\qquad \lambda,\delta>0.
\end{equation}
For constant stepsizes $\alpha_n=h>0$, this gives
\[
E_n\leq e^{-\lambda hn}E_0+
\frac{C}{\lambda}h^\delta(1-e^{-\lambda hn})
\leq e^{-\lambda hn}E_0+\frac{C}{\lambda}h^\delta.
\]
Thus a fixed step gives convergence to an $O(h^\delta)$ neighbourhood,
not convergence to zero. For
$\alpha_n=a/(n+n_0)^\vartheta$, where $a,n_0>0$ and $0<\vartheta<1$,
integral comparison yields
\[
t_n=\frac{a}{1-\vartheta}
\bigl\{(n+n_0)^{1-\vartheta}-n_0^{1-\vartheta}\bigr\}+O(1),
\]
so $e^{-\lambda t_n}\leq C_0e^{-c(n+n_0)^{1-\vartheta}}$ with
$c=\lambda a/(1-\vartheta)$. Moreover,
\[
\int_0^{t_n}e^{-\lambda(t_n-s)}\alpha(s)^\delta\,ds
=O(\alpha_n^\delta)
=O((n+n_0)^{-\vartheta\delta}).
\]
Indeed, on $s\geq t_n/2$ the piecewise-constant stepsize is
$O(\alpha_n)$, whereas the contribution from $s<t_n/2$ is exponentially
small. Consequently,
\[
E_n=O\!\left(e^{-c(n+n_0)^{1-\vartheta}}
+(n+n_0)^{-\vartheta\delta}\right).
\]
At the harmonic boundary, $\alpha_n=a/(n+n_0)$ and
$t_n=a\log(n+n_0)+O(1)$. Comparing the forcing with
$\alpha(s)^\delta\asymp e^{-\delta s/a}$ gives
\[
E_n=\begin{cases}
O(n^{-\delta}),&\lambda a>\delta,\\
O(n^{-\delta}\log n),&\lambda a=\delta,\\
O(n^{-\lambda a}),&\lambda a<\delta.
\end{cases}
\]
The threshold is specific to this harmonic boundary: for $0<\vartheta<1$,
contraction is stretched exponential in iteration time and eventually
dominates polynomial forcing, while in the harmonic case contraction and
forcing decay polynomially and their exponents compete.
\end{corollary}
The constant-step conclusion of
Corollary~\ref{cor:stepsize-rate-inheritance} is used here only in the
first-order setting with $\epsilon_m=0$; it makes no claim about a
constant-step result when the comparison has a persistent multiplicative
defect.

The applications illustrate different schedules. Strongly convex
SGD uses a decreasing subharmonic power law, while stochastic Heavy Ball uses
a constant step and contracts to an $O(h)$ neighbourhood. The transfer theorem
first gives a finite-time bound in diffusion time, and the learning-rate
schedule determines how that bound appears in iteration time.

\section{Stability from the associated SDE}
\label{sec:foster-lyapunov-stability-proof}
This section proves Proposition~\ref{prop:foster-lyapunov-stability-transfer} and records the power-lifting calculation used in Remark~\ref{rem:power-lifting-sde-stability}.

\begin{proof}
We first control the finite initial segment. Integrating the gradient
of $\Phi$ along the line from $\bm 0$ to $\bm x$ gives
\[
\begin{aligned}
\Phi(\bm x)
&=\Phi(\bm 0)+\int_0^1\nabla\Phi(\theta\bm x)^\top\bm x\,d\theta\\
&\leq\Phi(\bm 0)+C_\Phi\|\bm x\|
\int_0^1\polybrac{\theta\bm x}^{p-1}\,d\theta
\leq C\polybrac{\bm x}^{p}.
\end{aligned}
\]
The polynomial translation inequality and the primitive oracle bound
then give, for every $k$,
\[
\begin{aligned}
\mathbb E[\polybrac{\bm X_{k+1}}^p\mid\mathcal F_k]
&\leq C_p\left(\polybrac{\bm X_k}^p+
\alpha_k^p\mathbb E[\|\bm H_{k+1}\|^p\mid\mathcal F_k]\right)\\
&\leq C_p(1+M\alpha_k^p)\polybrac{\bm X_k}^p.
\end{aligned}
\]
The coercive bound in the proposition and
$\mathbb E\Phi(\bm X_0)<\infty$ first give
$\mathbb E\polybrac{\bm X_0}^p<\infty$. Since every $\alpha_k$ is finite,
induction in the conditional estimate therefore gives
$\mathbb E\polybrac{\bm X_k}^p<\infty$ for every $0\leq k\leq N$.
The polynomial upper bound on $\Phi$ gives
$\mathbb E\Phi(\bm X_k)<\infty$ on this finite prefix, in particular at
$k=N$. Thus no smallness condition is needed for the first $N$ steps.

We now start the Taylor argument at an index $n\geq N$ and write
$h=\alpha_n$. Taylor's formula with integral remainder gives
\[
\Phi(\bm X_n+h\bm H_{n+1})
=\Phi(\bm X_n)+h\nabla\Phi(\bm X_n)^\top\bm H_{n+1}
+h^2\int_0^1(1-\theta)\,
\bm H_{n+1}^\top\nabla^2\Phi(\bm X_n+\theta h\bm H_{n+1})
\bm H_{n+1}\,d\theta.
\]
The polynomial translation inequality
$\polybrac{\bm x+\bm y}^{r}\leq C_r(\polybrac{\bm x}^{r}+\|\bm y\|^r)$,
for $r\geq0$, and the Hessian bound imply that the conditional expectation
of the absolute remainder is at most
\[
C h^2\left[
\polybrac{\bm X_n}^{p-2}
\mathbb E(\|\bm H_{n+1}\|^2\mid\mathcal F_n)
+h^{p-2}\mathbb E(\|\bm H_{n+1}\|^p\mid\mathcal F_n)
\right].
\]
Conditional Jensen's inequality and the primitive $p$th-moment bound
give
\[
\mathbb E(\|\bm H_{n+1}\|^2\mid\mathcal F_n)
\leq\left(\mathbb E(\|\bm H_{n+1}\|^p\mid\mathcal F_n)\right)^{2/p}
\leq C\polybrac{\bm X_n}^{2}.
\]
Choose $\bar\alpha\leq1$. Since $h<\bar\alpha$, the remainder is
therefore bounded by
$Ch^2\polybrac{\bm X_n}^{p}$. Using
$\mathbb E[\bm H_{n+1}\mid\mathcal F_n]=\bm b(\bm X_n)$ and
$\polybrac{\bm x}^p\leq C_\Phi\Phi(\bm x)$, we obtain
\[
\mathbb E[\Phi(\bm X_{n+1})\mid\mathcal F_n]
\leq\Phi(\bm X_n)
+h\bm b(\bm X_n)^\top\nabla\Phi(\bm X_n)
+Ch^2\Phi(\bm X_n).
\]

The global Lipschitz bound on $\sigma$ implies
$\|\sigma(\bm x)\|_F\leq C\polybrac{\bm x}$. Hence
\[
|A(\bm x):\nabla^2\Phi(\bm x)|
\leq C\polybrac{\bm x}^{2}\polybrac{\bm x}^{p-2}
\leq C\Phi(\bm x).
\]
For $n\geq N$, we have $t_n\geq t_N\geq t_{N_0}$ and
$\alpha(t_n)=\alpha_n=h$. The tail generator inequality therefore yields
\[
\bm b(\bm x)^\top\nabla\Phi(\bm x)
=\mathcal L_{t_n}\Phi(\bm x)
-\frac h2 A(\bm x):\nabla^2\Phi(\bm x)
\leq-\gamma\Phi(\bm x)+K+Ch\Phi(\bm x).
\]
Substitution into the Taylor estimate gives
\[
\mathbb E[\Phi(\bm X_{n+1})\mid\mathcal F_n]
\leq\bigl(1-\gamma h+Ch^2\bigr)\Phi(\bm X_n)+Kh.
\]
Decrease $\bar\alpha$ if necessary so that
$C\bar\alpha\leq\gamma/2$ and $\bar\alpha\leq2/\gamma$. Then for
$n\geq N$,
\[
\mathbb E[\Phi(\bm X_{n+1})\mid\mathcal F_n]
\leq\left(1-\frac{\gamma h}{2}\right)\Phi(\bm X_n)+Kh,
\qquad 0\leq1-\frac{\gamma h}{2}\leq1.
\]
Taking expectations, set $u_n=\mathbb E\Phi(\bm X_n)$ and
$L=2K/\gamma$. For $n\geq N$,
\[
u_{n+1}\leq\left(1-\frac{\gamma\alpha_n}{2}\right)u_n
+\frac{\gamma\alpha_n}{2}L.
\]
The coefficient lies in $[0,1]$, so induction from $N$ gives
$u_n\leq\max\{u_N,L\}$ for every $n\geq N$. Combining this tail bound
with the finite prefix yields
\[
\sup_{n\geq0}\mathbb E\Phi(\bm X_n)
\leq\max\left\{\max_{0\leq k\leq N}\mathbb E\Phi(\bm X_k),
\frac{2K}{\gamma}\right\}<\infty.
\]
Finally, the one-sided coercive bound gives
$M_p:=\sup_{n\geq0}\mathbb E\polybrac{\bm X_n}^{p}
\leq C_\Phi\sup_{n\geq0}\mathbb E\Phi(\bm X_n)<\infty$.
\end{proof}

\begin{lemma}[Power lifting]
\label{lem:power-lifting-sde-stability}
Let $W\in C^2(\mathbb R^d)$ satisfy $W\geq0$, and let
$2\leq p_0\leq p$. Suppose, for positive finite constants $c_-$, $c_+$,
and $C_W$, that
\[
c_-\polybrac{\bm x}^{p_0}\leq1+W(\bm x)
\leq c_+\polybrac{\bm x}^{p_0},\qquad
\|\nabla W(\bm x)\|\leq C_W\polybrac{\bm x}^{p_0-1},\qquad
\|\nabla^2W(\bm x)\|\leq C_W\polybrac{\bm x}^{p_0-2}.
\]
Assume that $\sigma$ is globally Lipschitz, as in
Assumption~\ref{assum:v3-primitive-finite-window}, and that on a tail
\[
\mathcal L_tW(\bm x)\leq-\lambda_0W(\bm x)+K_0
\]
for some $\lambda_0>0$ and finite $K_0$. Set $r=p/p_0$ and
$\Phi=(1+W)^r$. Then:
\begin{enumerate}
\item[\rm(a)] $\Phi(\bm x)\asymp\polybrac{\bm x}^p$;
\item[\rm(b)] for a finite constant $C$,
\[
\|\nabla\Phi(\bm x)\|\leq C\polybrac{\bm x}^{p-1},\qquad
\|\nabla^2\Phi(\bm x)\|\leq C\polybrac{\bm x}^{p-2};
\]
\item[\rm(c)] there are $c,C,\bar\alpha>0$ such that, wherever the
lower-order Foster--Lyapunov inequality holds and $\alpha(t)\leq\bar\alpha$,
\[
\mathcal L_t\Phi(\bm x)\leq-c\Phi(\bm x)+C.
\]
\end{enumerate}
\end{lemma}

\begin{proof}
The two-sided growth assumption gives
$\Phi\asymp\polybrac{\bm x}^{p}$. Differentiation yields
\[
\nabla\Phi=r(1+W)^{r-1}\nabla W,\qquad
\nabla^2\Phi=r(1+W)^{r-1}\nabla^2W
+r(r-1)(1+W)^{r-2}\nabla W\nabla W^\top.
\]
The growth assumptions give the bounds in (b). For $1\leq r<2$,
the negative exponent $r-2$ in the Hessian's second term is bounded using
the lower coercivity bound on $1+W$. The generator chain rule gives
\[
\mathcal L_t\Phi
=r(1+W)^{r-1}\mathcal L_tW
+\frac{\alpha(t)}2r(r-1)(1+W)^{r-2}
\|\sigma^\top\nabla W\|^2.
\]
Writing $z=1+W$, the assumed drift bound gives
$r z^{r-1}\mathcal L_tW\leq-r\lambda_0z^r+Cz^{r-1}
\leq-c_0z^r+C$ by Young's inequality. Global Lipschitzness gives
$\|\sigma(\bm x)\|_F\leq C\polybrac{\bm x}$. Combining this with the
gradient bound and the lower coercivity bound shows
\[
(1+W)^{r-2}\|\sigma^\top\nabla W\|^2\leq C(1+W)^r=C\Phi.
\]
Consequently,
$\mathcal L_t\Phi\leq-c_0\Phi+C+C\alpha(t)\Phi$. Choosing
$\bar\alpha$ small enough to absorb the last term proves (c).
\end{proof}

\end{document}